\documentclass[11pt, onecolumn, journal,compsoc]{IEEEtran}

\usepackage[T1]{fontenc}
\usepackage{amsmath, amssymb,amsthm}
\usepackage{stmaryrd}
\usepackage{epsfig,calc, graphicx}
\usepackage{wasysym,pifont}
\usepackage[usenames]{color}
\usepackage{multicol}
\usepackage{todonotes}
\usepackage{dsfont}
\usepackage{tikz}
\usepackage{pgfplots}
\usepackage{multirow}
\usepackage{rotating}
\usepackage{url}
\usepackage{hhline}
\usepackage{ifthen}
\usepackage{dblfloatfix}

\newcommand*\rot{\rotatebox{90}}

\usetikzlibrary{external}
\tikzset{external/mode=graphics if exists}
\pgfplotsset{scaled x ticks=false}

\newboolean{onecol}
\setboolean{onecol}{true}

\usetikzlibrary{matrix}
\usepgfplotslibrary{groupplots}

\title{Online and Stable Learning of Analysis Operators}

\author{Michael Sandbichler \& Karin Schnass
\thanks{Both authors are with the Department of Mathematics, University of Innsbruck, Technikerstra\ss e 13, 6020 Innsbruck, Austria. Email: michael.sandbichler@uibk.ac.at, karin.schnass@uibk.ac.at}
}

\newcommand{\cosupp}{\mathrm{cosupp}}

\newcommand\ip[2]{\langle #1, #2\rangle}

\newcommand\eps{\varepsilon}

\newcommand\dico{\Phi}

\newcommand\atom{\phi}

\newcommand\bGamma{\bar\Gamma}
\newcommand\bgamma{\bar\gamma}
\newcommand\pdico{\Psi}

\newcommand\patom{\psi}

\newcommand\tr{\operatorname{tr}}

\newcommand{\R}{{\mathbb{R}}}

\newcommand{\E}{{\mathbb{E}}}

\renewcommand{\S}{\mathbb{S}}

\DeclareMathOperator*{\argmin}{\text{arg\,min}}

\newcommand{\supp}{\operatorname{supp}}
\newcommand{\colspan}{\operatorname{colspan}}
\newcommand{\rowspan}{\operatorname{rowspan}}

\newcommand{\id}{\mathds{1}}

\newcommand{\Ocal}{\mathcal{O}}
\newcommand{\Ucal}{\mathcal{U}}
\newcommand{\Acal}{\mathcal{A}}

\newcommand{\Xcal}{\mathcal{X}}
\newcommand{\Ncal}{\mathcal{N}}

\theoremstyle{plain}
\newtheorem{theorem}{Theorem}[section]

\theoremstyle{remark}

\begin{document}

\maketitle

\begin{abstract}
In this paper four iterative algorithms for learning analysis operators are presented. They are built upon the same optimisation principle underlying both Analysis K-SVD and Analysis SimCO. 
The Forward and Sequential Analysis Operator Learning (FAOL and SAOL) algorithms are based on projected gradient descent with optimally chosen step size. The Implicit AOL (IAOL) algorithm is inspired by the implicit Euler scheme for solving ordinary differential equations and does not require to choose a step size. The fourth algorithm, Singular Value AOL (SVAOL), uses a similar strategy as Analysis K-SVD while avoiding its high computational cost. All algorithms are proven to decrease or preserve the target function in each step and a characterisation of their stationary points is provided. Further they are tested on synthetic and image data, compared to Analysis SimCO and found to give better recovery rates and faster decay of the objective function respectively. In a final denoising experiment the presented algorithms are again shown to perform similar to or better than the state-of-the-art algorithm ASimCO. 
\end{abstract}

%\begin{keywords}
%\noindent analysis operator learning, analysis dictionary learning, online learning, cosparse, sequential, stochastic gradient descent, thresholding, denoising
%\end{keywords}
%
%%%%%%%%%%%%%%
\section{Introduction}\label{sec:intro}
%%%%%%%%%%%%%%

%%% Intro 
Many tasks in high dimensional signal processing, such as denoising or reconstruction from incomplete information, can be efficiently solved if the data at hand is known to have intrinsic low dimension. One popular model with intrinsic low dimension is the union of subspaces model, where every signal is assumed to lie in one of the low dimensional linear subspaces. However, as the number of subspaces increases, the model becomes more and more cumbersome to use unless the subspaces can be parametrised. Two examples of large unions of parametrised subspaces, that have been successfully employed, are sparsity in a dictionary and cosparsity in an analysis operator. In the sparse model the subspaces correspond to the linear span of just a few normalised columns, also known as atoms, from a $d\times K$ dictionary matrix, $\dico=(\atom_1\ldots \atom_K)$ with $\|\atom_k\|_2=1$, meaning, any data point $y$ can be approximately represented as superposition of $S\ll d$ dictionary elements. If we denote the restriction of the dictionary to the atoms/columns indexed by $I$ as $\dico_I$, we have
\begin{align*}
y \in \bigcup_{|I|\leq S} \colspan \dico_I,  \quad \mbox{or} \quad y \approx \dico x, \quad \mbox{with } x  \mbox{ sparse}.
\end{align*}
In the cosparse model the subspaces correspond to the orthogonal complement of the span of some normalised rows, also known as analysers, from a $K\times d$ analysis operator $\Omega=(\omega_1^\star \ldots \omega_K^\star)^\star$ with $\|\omega_k\|_2=1$. This means that any data point $y$ is orthogonal to $\ell$ analysers or in other words that the vector $\Omega y$ has $\ell$ zero entries and is sparse. If we denote the restriction of the analysis operator to the analysers/rows indexed by $J$ as $\Omega_J$, we have
\begin{align*}
y \in \bigcup_{|J|\geq \ell} (\rowspan \Omega_J)^\perp,  \quad \mbox{or} \quad \Omega y \approx z, \quad \mbox{with } z  \mbox{ sparse}.
\end{align*}
Note that in this model it is not required that the signals lie in the span of $\Omega^\star$, in particular $\Omega^\star \Omega$ need not be invertible. Before being able to exploit these models for a given data class, it is necessary to identify the parametrising dictionary or analysis operator. This can be done either via a theoretical analysis or a learning approach. While dictionary learning is by now an established field, see \cite{rubrel10} for an introductory survey, results in analysis operator learning are still countable, \cite{yanagrda11, rabr13,nadaelgr13, yanagrda13,hakldi13,rupeel13,dowada14,ekba14,ginaelgr14,sewogrkl16,dowadaplha16}. \\
Most algorithms are based on the minimisation of a target function together with various constraints on the analysis operator. In~\cite{yanagrda11,yanagrda13}, an $\ell_1$-norm target function together with the assumption that the operator is a unit norm tight frame was used. 
In transform learning~\cite{rabr13}, in addition to the sparsity of $\Omega Y$, a regularisation term amounting to the negative log-determinant of $\Omega$ is introduced to enforce full rank of the operator. Note that due to the nature of this penalty term overcomplete operators cannot be considered.
The geometric analysis operator learning framework~\cite{hakldi13} uses operators with unit norm rows, full rank and furthermore imposes the restriction that none of the rows are linearly dependent. The last assumption is an additional restriction only if the considered operators are overcomplete. These assumptions admit a manifold structure and a Riemannian gradient descent algorithm is used in order to find the operator.
Analysis K-SVD~\cite{rupeel13} uses an analysis sparse coding step to find the cosupports for the given data and then computes the singular vector corresponding to the smallest singular value of a matrix computed from the data.
Finally, in~\cite{dowada14,dowadaplha16} a projected gradient descent based algorithm with line search, called Analysis SimCO, is presented. There, the only restriction on the analysis operator is that its rows are normalised and the target function enforces sparsity of $\Omega Y$.\\
%%%
{\bf Contribution:} In this work we will contribute to the development of the field by developing four algorithms for learning analysis operators, which  improve over state of the art algorithms such as Analysis K-SVD, \cite{rupeel13} and Analysis SimCo, \cite{dowada14, dowadaplha16}, in terms of convergence speed, memory complexity and performance. \\
%%% Organisation
{\bf Outline:} The paper is organised as follows. After introducing the necessary notation, in the next section we will remotivate the optimisation principle that is the starting point of A-KSVD and ASimCO and shortly discuss the advantages and disadvantages of the two algorithms. We then take a gradient descent approach similar to ASimCO, replacing the line search with an optimal choice for the step size, resulting in the Forward Analysis Operator Learning algorithm (FAOL). In order to obtain an online algorithm, which processes the training data sequentially, we devise an estimation procedure for the quantities involved in the step size calculation leading to the Sequential Analysis Operator Learning algorithm (SAOL) and test the presented algorithms both on synthetic and image data. Inspired by the implicit Euler scheme for solving ordinary differential equations and the analysis of some special solutions of these equations, in Section~\ref{sec:backward}, we gradually invest in the memory requirements and computational complexity per iteration of our schemes in return for avoiding the stepsize altogether and overall faster convergence, leading to Implicit (IAOL) and Singular Vector Analysis Operator Learning (SVAOL). After testing the new algorithms on synthetic and image data demonstrating the improved recovery rates and convergence speed with respect to ASimCO, in Section~\ref{sec:denoise} we apply them to image denoising again in comparison to ASimCO. Finally, in the last section we provide a short discussion of our results and point out future directions of research.\\
%%% Notation
{\bf Notation:} Before hitting the slopes, we summarise the notational conventions used throughout this paper. The operators $\Omega$ and $\Gamma$ will always denote matrices in $\R^{K\times d}$ and for a matrix $A$ we denote its transpose by $A^\star$. More specifically, we will mostly consider matrices in the manifold $\Acal := \{\Gamma\in\R^{K\times d}\colon \forall k\in[K]\colon \|\gamma_k\|_2 = 1\}$, where $\gamma_k$ denotes the $k$-th row of the matrix $\Gamma$. By $[n]$, we denote the set $\{1,2,\ldots, n\}$ and we adopt the standard notation $|M|$ for the cardinality of a set $M$. By $\Gamma_J$ with $J\subset [K]$ we denote the restriction of $\Gamma$ to the rows indexed by $J$.\\
A vector $y\in \R^d$ is called $\ell$-cosparse with respect to $\Omega$, if there is an index set $\Lambda\subset [K]$ with $|\Lambda|=\ell$, such that $\Omega_\Lambda y = 0$. The support of a vector $x\in\R^K$ is defined as $\supp(x) = \{k\in[K]\colon x_k\neq0\}$ and the cosupport accordingly as $\cosupp(x) = \{k\in[K]\colon x_k=0\}$. Note that by definition we have $\supp(x) \cup \cosupp(x) =[K]$.
For the runtime complexity $R(n)$, we adopt standard Landau notation, i.e. $R(n)=\Ocal(f(n))$ means, there is a constant $C>0$, such that for large $n$, the runtime $R(n)$ satisfies $R(n)\leq Cf(n)$. \\
Finally, the Frobenius norm of a matrix $A$ is defined by $\Vert A\Vert_F^2:= \tr(A^\star A)$.
%%%%%%%%%%%%%%
\section{Two explicit analysis operator learning algorithms - FAOL and SAOL}\label{sec:cheap}
%%%%%%%%%%%%%%
Since optimisation principles have already successfully led to online algorithms for dictionary learning, \cite{sc14b, sc15}, we will start our quest for an online algorithm by motivating a suitable optimisation principle for analysis operator learning. Suppose, we are given signals $y_n\in \R^d$ that are perfectly cosparse in an operator $\Omega$, i.e.
$\Omega y_n$ has $\ell$ zero entries or equivalently $\Omega y_n - x_n = 0$ for some $x_n$ which has $K-\ell$ non-zero entries. 
If we collect the signals $y_n$ as columns in the matrix $Y=(y_1 \ldots y_N)$, then by construction we have $\Omega Y - X=0$ for some $X \in \Xcal_\ell$ with $\Xcal_\ell:= \{(x_1,x_2,\ldots,x_N)\in \R^{K\times N} \colon |\supp(x_n)| = K-\ell\}$. In the more realistic scenario, where the signals are not perfectly cosparse, we should still have $\Omega Y - X\approx 0$, which naturally leads to the following minimisation program to recover~$\Omega$,
\vspace{-.05cm}\begin{align}\label{eq:AOLcheap2}
%\hat \Omega = 
\argmin_{\Gamma \in \Acal, X\in \Xcal_\ell} \| \Gamma Y - X\|_F^2.
\end{align}
Apart from additional side constraints on $\Gamma$, such as incoherence, the optimisation program above has already been used successfully as starting point for the development of two analysis operator learning algorithms, Analysis K-SVD~\cite{rupeel13} and Analysis SimCO~\cite{dowada14, dowadaplha16}. AKSVD is an alternating minimisation algorithm, which alternates between finding the best $X \in \Xcal_\ell$ for the current $\Gamma$ and updating $\Gamma$ based on the current $X$. The cosparse approximation scheme used there is quite cumbersome and costly, which means that the algorithm soon becomes intractable as $d$ increases.
ASimCO is a (gradient) descent algorithm with line search. It produces results similar to AKSVD and has the advantage that it does so with a fraction of the computational cost. Still, at closer inspection we see that the algorithm has some problematic aspects. 
The line search cannot be realised resource efficiently, since in each step several evaluations of the target function are necessary, which take up a lot of computation time. Moreover, for each of these function evaluations we must either reuse the training data, thus incurring high storage costs, or use a new batch of data, thus needing a huge amount of training samples. 
Still, if we consider the speedup of ASimCO with respect to AKSVD we see that gradient descent is a promising approach if we can avoid the line search and its associated problems. \\
To see that a gradient descent based algorithm for our problem can also be sequential, let us rewrite our target function, $g_N(\Gamma) = \min_{X\in \Xcal_\ell} \| \Gamma Y - X\|_F^2$.
Abbreviating $\Lambda_n=\supp(x_n)$ and $\Lambda^c_n=\cosupp(x_n)$, we have
\begin{align*}
g_N (\Gamma) &=\sum_{n=1}^N \min_{x_n:|\Lambda_n|=K-\ell}\Vert \Gamma y_n - x_n\Vert_2^2 =\\
& =  \sum_{n=1}^N \min_{x_n: |\Lambda_n|=K-\ell}(\Vert \Gamma_{\Lambda^c_n} y_n\Vert_2^2 + \underbrace{\Vert\Gamma_{\Lambda_n} y_n- x_n\Vert_2^2}_{=0})\\
&=\sum_{n=1}^N \min_{|J| = \ell} \Vert \Gamma_J y_n\Vert_2^2=: f_N(\Gamma).
\end{align*}
Since the gradient of a sum of functions is the sum of the gradients of these functions, from $f_N$ we see that the gradient of our objective function can be calculated in an online fashion.\\ 
Before going into more details about how to avoid a line search and stay sequential, let us lose a few words about the uniqueness of the minima of our objective function.\\
If the signals are perfectly cosparse in $\Omega$, clearly there is a global minimum of $f_N$ at $\Omega$. However, one can easily see that all permutations and sign flips of rows of $\Omega$ are also minimisers of $f_N$. We call these the \emph{trivial ambiguities}. The more interesting question is whether there are other global or local minima? \\
This question is best answered with an example. Assume that all our training signals are (perfectly) $\ell$-cosparse in $\Omega$ but lie in a subspace of $\R^d$. In this case we can construct a continuum of operators $\Gamma$, which also satisfy $f_N(\Gamma)=0$ by choosing a vector $v$ with $\|v\|_2=1$ in the orthogonal complement of this subspace, and by setting $\gamma_k = a_k \omega_k + b_k v$ for some $a^2_k + b^2_k=1$. 
This example indicates that isotropy in the data is important for our problem to be well posed. On the other hand, in case the data has such a low dimensional structure, which can be found via a singular value decomposition of $Y^\star Y$, it is easy to transform the ill posed problem into a well posed one. Armed with the non-zero singular vectors, we just have to project our data onto the lower dimensional space spanned by these vectors and learn the analysis operator within this lower dimensional space. In the following, we assume for simplicity that any such preprocessing has already been done and that the data isotropically occupies the full ambient space $\R^d$ or equivalently that $Y^\star Y$ is well conditioned.

%%%%%%%%%%%%%%%%
\subsection{Minimising $f_N$}
%%%%%%%%%%%%%%%%

As mentioned above in order to get an online algorithm we want to use a gradient descent approach but avoid the line search. Our strategy will be to use projected stochastic gradient-type descent with carefully chosen stepsize. Given the current estimate of the analysis operator $\Gamma$, one step of (standard) gradient descent takes the form 
\[\bGamma= \Gamma- \alpha \nabla f_N\left(\Gamma \right).\]
Let us calculate the gradient $\nabla f_N\left(\Gamma \right)$ wherever it exists.
Denote by $J_n$ the set\footnote{The careful reader will observe that the set $J_n$ might not be unique for every $y_n$ and $\Gamma$. If for a given $\Gamma$ at least one $J_n$ is not uniquely determined and $\min_{|J| = \ell} \Vert \Gamma_J y_n \Vert_2^2>0$, then the target function is not differentiable in $\Gamma$. For simplicity we will continue the presentation as if the $J_n$ where uniquely determined, keeping in mind that the derived descent direction only coincides with the gradient where it exists.} for which $\Vert \Gamma_{J_n} y_n\Vert_2^2 = \min_{|J| = \ell} \Vert \Gamma_J y_n \Vert_2^2$, then the derivative of $f_N$ with respect to a row $\gamma_k$ of $\Gamma$ is

%%%%%%%
\ifthenelse{\boolean{onecol}}{\begin{align}\label{eq:gradient}
\frac{\partial f_N}{\partial \gamma_k} (\Gamma) = \sum_{n=1}^N \sum_{j\in J_n}\frac{\partial}{\partial \gamma_k} \ip{\gamma_j}{y_n}^2 = \sum_{n=1}^N \sum_{j\in J_n}2 \ip{\gamma_j}{y_n}y_n^\star \delta_{kj} = \sum_{n\colon k\in J_n} 2 \ip{\gamma_k}{y_n}y_n^\star=:2 g_k.
\end{align}
}{	
\begin{align}\label{eq:gradient}
\frac{\partial f_N}{\partial \gamma_k} (\Gamma) &= \sum_{n=1}^N \sum_{j\in J_n}\frac{\partial}{\partial \gamma_k} \ip{\gamma_j}{y_n}^2 \notag\\
%&= \sum_{n=1}^N \sum_{j\in J_n}2 \ip{\gamma_j}{y_n}y_n^\star \delta_{kj}\notag \\
&= \sum_{n\colon k\in J_n} 2 \ip{\gamma_k}{y_n}y_n^\star=:2 g_k. 
\end{align}
}
Note that as expected the vectors $g_k$ can be calculated online, that is given a continuous stream of data $y_n$, we compute $J_n$, update all $g_k$ for $k\in J_n$, and forget the existence of $y_n$. After processing all signals, we set
\begin{align}
\bgamma_k = \left(\gamma_k - \alpha_k g_k\right)\beta_k.
\end{align}
where $\beta_k = \Vert \gamma_k - \alpha_k g_k\Vert_2^{-1}$ is a factor ensuring normalisation of $\bgamma_k$. This normalisation corresponds to a projection onto the manifold $\Acal$ and is necessary, since a standard descent step will most likely take us out of the manifold.
If we compare to dictionary learning, e.g. \cite{sc14b}, it is interesting to observe that we cannot simply choose $\alpha_k$ by solving the linearised optimisation problem with side constraints using Lagrange multipliers, since this would lead to a zero-update $\bgamma_k =0$. \\
In order to find the correct descent parameter, note that the current value of the target function is given by 
\[f_N(\Gamma) = \sum_{n=1}^N \sum_{k\in J_n} |\ip{\gamma_k}{y_n}|^2 = \sum_{k=1}^K \sum_{n\colon k\in J_n} |\ip{\gamma_k}{y_n}|^2.\]
Defining $A_k:= \sum_{n\colon k\in J_n} y_n y_n^\star$, we see that 
$f_N(\Gamma) = \sum_{k=1}^K \gamma_k A_k \gamma_k^\star$ and we can optimally decrease the objective function by choosing $\alpha_k$, such that it minimises $\bgamma_k A_k \bgamma_k^\star$. Note also that with this definition, the descent directions $g_k$ defined in Equation~\eqref{eq:gradient} are given by $g_k = \gamma_k A_k$. \\
First assume that $g_k \neq 0$, or more generally $g_k \neq \lambda_k \gamma_k$. In case $g_k = 0$ the part of the objective function associated to $\gamma_k$ is already zero and cannot be further reduced, while in case $g_k = \lambda_k \gamma_k$ any admissible stepsize not leading to the zero vector preserves the current analyser, that is $\bar \gamma_k = \gamma_k$. To optimally decrease the target function, we need to solve
\begin{equation}\label{eq:alphaCondition}
\alpha_k = \argmin_{\alpha} \frac{(\gamma_k -\alpha g_k)A_k(\gamma_k - \alpha g_k)^\star}{\|(\gamma_k -\alpha g_k)\|^2}.
\end{equation}
Defining $a_k =\gamma_k A_k\gamma_k^\star$, $b_k = \gamma_k A_k^2\gamma_k^\star$ and $c_k = \gamma_k A_k^3\gamma_k^\star$ a short computation given in Appendix~\ref{sec:CompAlpha} shows that whenever $b_k^2 \neq a_kc_k$ the optimal stepsize has the form,
\begin{align*}
\alpha_k = \frac{a_kb_k-c_k+\sqrt{(c_k-a_kb_k)^2 - 4(b_k^2-a_kc_k)(a_k^2-b_k)}}{2(b_k^2-a_kc_k)}.
\end{align*}
If $b_k^2 = a_kc_k$ and $b_k\neq 0$, the optimal stepsize is $\alpha_k = \frac{a_k}{b_k}$.
Finally, if $b_k =\|A_k\gamma_k^\star\|_2^2= 0$ it follows that $A_k\gamma_k^\star = 0$ and therefore also $a_k = c_k = 0$. In this case we set $\alpha_k = 0$, as $\gamma_k A_k\gamma_k^\star$ is already minimal.\\
We summarise the first version of our derived algorithm, called Forward Analysis Operator Learning (FAOL) in Table~\ref{tab:FAOL}. As input parameters, it takes the current estimate of the analysis operator $\Gamma \in \R^{K\times d}$, the cosparsity parameter $\ell$ and $N$ training signals $Y=(y_1,y_2,\ldots, y_N)$.
\begin{table}
	\centering
	\noindent\fbox{%
		\parbox{0.465\textwidth}{%
			{\bf FAOL($\Gamma,\ell,Y$) - (one iteration)} \\
			
			\begin{itemize}
			\item For all $n\in[N]$:
				\begin{itemize}
					\item Find $J_n = \argmin_{|J| = \ell} \Vert \Gamma_J y_n\Vert_2^2$.
					\item For all $k\in[K]$ update $A_k = A_k+y_n y_n^\star$ if $k\in J_n$.
				\end{itemize}
			\item For all $k\in[K]$:
				\begin{itemize}
					\item Set $a = \gamma_k A_k \gamma_k^\star$, $b = \gamma_k A_k^2 \gamma_k^\star$ and $c = \gamma_k A_k^3 \gamma_k^\star$.
					\item If $b^2-ac\neq 0$, set $\alpha_k:= \frac{ab-c+\sqrt{(c-ab)^2 - 4(b^2-ac)(a^2-b)}}{2(b^2-ac)}$.
					\item If $b^2-ac= 0$ and $b\neq 0$, set $\alpha_k:=\tfrac{a}{b}$.
					\item If $b^2-ac= 0$ and $b= 0$, set $\alpha_k := 0$.
					\item Set $\bgamma_k = \gamma_k(\id  - \alpha_k A_k)$.
				\end{itemize}
			\end{itemize}
		Output $\bGamma =(\frac{\bgamma_1}{\Vert \bgamma_1\Vert_2},\ldots,\frac{\bgamma_K}{\Vert \bgamma_K\Vert_2})^\star$.
		}
	}
	\caption{The FAOL algorithm}
	\label{tab:FAOL}
\end{table}
As a result of the optimal stepsize choice we can prove the following theorem characterising the behaviour of the FAOL algorithm.

\begin{theorem}\label{thm:FAOLdecrease}
The FAOL algorithm decreases or preserves the value of the target function in each iteration.\\ Preservation rather than decrease of the target function can only occur if all rows $\gamma_k$ of the current iterate $\Gamma$ are eigenvectors of the matrix $A_k(\Gamma)$.
\end{theorem}

\begin{proof}
To prove the first part of the theorem observe that
%%% one vs twocol
\ifthenelse{\boolean{onecol}}{
\[f(\Gamma) = \sum_{k=1}^K \sum_{n\colon k\in J_n} |\ip{\gamma_k}{y_n}|^2\geq \sum_{k=1}^K \sum_{n\colon k\in J_n} |\ip{\bgamma_k}{y_n}|^2\geq \sum_{k=1}^K \sum_{n\colon k\in \bar J_n} |\ip{\bgamma_k}{y_n}|^2 = f(\bGamma).\]}
{
\begin{align*}
f_N(\Gamma) & = \sum_{k=1}^K \sum_{n\colon k\in J_n} |\ip{\gamma_k}{y_n}|^2 \geq \sum_{k=1}^K \sum_{n\colon k\in J_n} |\ip{\bgamma_k}{y_n}|^2\\
&\geq \sum_{k=1}^K \sum_{n\colon k\in \bar J_n} |\ip{\bgamma_k}{y_n}|^2 = f_N(\bGamma),
\end{align*}
}
where $\bar J_n$ denotes the minimising set for $y_n$ based on $\bar \Gamma$.
The first inequality follows from the choice of $\alpha_k$ and the second inequality follows from the definition of the sets $J_n$ and $\bar J_n$. \\
The second part of the theorem is a direct consequence of the derivation of $\alpha_k$ given in Appendix~\ref{sec:CompAlpha}.
\end{proof}
Let us shortly discuss the implications of Theorem~\ref{thm:FAOLdecrease}. It shows that the sequence of values of the target function $v_k = f(\Gamma^{(k)})$ converges. This, however, does not imply convergence of the algorithm as suggested in~\cite{dowadaplha16}, at least not in the sense that the sequence $\Gamma^{(k)}$ converges. Indeed the sequence $\Gamma^{(k)}$ could orbit around the set $\mathcal{L} = \{\Gamma\in\Acal\colon f(\Gamma) = v\}$, where $v = \lim_{k\to\infty} v_k$. 
If this set contains more than one element, there need not exist a limit point of the sequence $\Gamma^{(k)}$. Nevertheless, due to compactness of the manifold $\Acal$, we can always find a subsequence, that converges to an element $\Gamma\in\mathcal{L}$. In order to avoid getting trapped in such orbital trajectories, in numerical experiments we draw a fresh batch of signals $y_1,\ldots, y_N$ in each iteration of the algorithm.
\\
We proceed with an analysis of the runtime complexity of the FAOL algorithm.
The cost of finding the support sets $S_k = \{n\colon k\in J_n\}$ of average size $N\ell/K$ in the FAOL algorithm is $\Ocal(dKN)$ amounting to the multiplication of the current iterate $\Gamma$ with the data matrix $Y$ and subsequent thresholding.
We can now either store the $K$ matrices $A_k$ of size $d\times d$ amounting to a memory complexity of $\Ocal(Kd^2)$ or store the data matrix $Y$ and the optimal cosupports $S_k$ requiring memory on the order of $\Ocal(dN)$ and $\Ocal(\ell N)$, respectively. Setting up all matrices $A_k$   takes $\Ocal(d^2N)$ multiplications and $\Ocal(\ell d^2N)$ additions, if done sequentially, and dominates the cost of calculating $a_k, b_k,c_k$.
Denote by $Y_k$ the submatrix of the data matrix $Y$ with columns indexed by $S_k$. Note that with this convention we have $A_k = Y_k Y_k^\star$. If we store the data matrix $Y$ and the sets $S_k$, we can also compute all necessary quantities via
$g_k = (\gamma_k Y_k) Y_k^\star$, $a = \ip{g_k}{\gamma_k}$, $b = \ip{g_k}{g_k}$ and $c = \ip{g_k Y_k}{g_k Y_k}$ altogether amounting to $\Ocal(\ell dN)$ floating point operations, as in this case only matrix-vector products have to be computed. So while the memory complexity of the first approach might be smaller depending on the amount of training data, the second approach has a reduced computational complexity.\\
If we are now given a continuous stream of high dimensional data, it is not desirable to store either the matrices $A_k$ or the data matrix $Y$, so as a next step we will reformulate the FAOL algorithm in an online fashion. Note, that with the exception of $c$, all quantities in the FAOL algorithm can be computed in an online fashion.
We will solve this issue by estimating $c$ from part of the data stream. First, note that if we exchange the matrix $A_k$ in the FAOL algorithm with the matrix $\tilde A_k :=\frac{1}{|S_k|} \sum_{n\in S_k} y_n y_n^\star$, where $S_k:= \{n\in[N]\colon k\in J_n\}$, we do not alter the algorithm.
The numbers $c_k$ can be computed from the gradients $g_k$ and the matrix $A_k$ via $c_k = g_k A_k g_k^{\star}=\tfrac{1}{|S_k|}\sum_{n\in S_k} |\ip{g_k}{y_n}|^2$. If we want to estimate $c_k$, we need both, a good estimate of the gradients $g_k$, and a good estimate of $A_k g_k^\star$.  
We do this by splitting the datastream into two parts. The first part of the datastream is used to get a good estimate of the normalised gradient $g_k$. The second part is used to refine $g_k$ as well as to estimate $ \tfrac{1}{|S_k|}\sum_{n\in S_k} |\ip{g_k}{y_n}|^2$. The parameter $\eps$ specifies the portion of the datastream used to estimate $c_k$ and refine $g_k$. We summarise all our considerations leading to the algorithm, referred to as Sequential Analysis Operator Learning (SAOL), in Table~\ref{tab:SAOL}.\\
\begin{table}
	\centering
	\noindent\fbox{%
		\parbox{0.465\textwidth}{%
			{\bf SAOL($\Gamma,\ell,Y,\eps$) - (one iteration)} \\
			
			Initialize $I_k,C_k,c_k = 0$ and $g_k = 0$ for $k\in [K]$.
			\begin{itemize}
				\item For all $n\in[N]$:
				\begin{itemize}
					\item Find $J_n = \argmin_{|J| = \ell} \Vert \Gamma_J y_n\Vert_2^2$.
					\item For all $k\in J_n$ update $I_k \rightarrow I_k+1$ 
					and \[g_k\rightarrow \frac{I_k-1}{I_k}g_k+\frac{1}{I_k}\ip{\gamma_k}{y_n}y_n^\star.\]
					\item If $n > (1-\eps)N$ and $k\in J_n$  update $C_k  \rightarrow C_k+1$ and 
					\[ c_k \rightarrow \frac{C_k-1}{C_k}c_k+\frac{1}{C_k} |\ip{g_k}{y_n}|^2.\]
				\end{itemize}
				\item For all $k\in [K]$:
				\begin{itemize}
				\item Set $a =\ip{\gamma_k}{g_k}$, $b = \ip{g_k}{g_k}$ and $c=  c_k$, 
				\item Set $\alpha_k = \frac{ab-c+\sqrt{(c-ab)^2 - 4(b^2-ac)(a^2-b)}}{2(b^2-ac)}$.
				\item Set $\bgamma_k = \left(\gamma_k - \alpha_k g_k\right)$.
				
				\end{itemize}
			\item Output $\bGamma =(\frac{\bgamma_1}{\Vert \bgamma_1\Vert_2},\ldots,\frac{\bgamma_K}{\Vert \bgamma_K\Vert_2})^\star$.
			\end{itemize}
		}
 	}
	\caption{The SAOL algorithm}
	\label{tab:SAOL}
\end{table}
Concerning the computation and storage costs, we see that, as for FAOL, the computationally expensive task is determining the sets $J_n$. This has to be done for each of our $N$ sample vectors via determining the $\ell$ smallest entries in the product $\Gamma y_n$. The matrix-vector product takes $(2d-1)K$ operations and searching can be done in one run through the $K$ resulting entries, yielding an overall runtime complexity of $\Ocal(dKN)$. However, compared to FAOL, the sequential version has much lower memory requirements on the order of $\Ocal(dK)$, corresponding to the gradients $g_k$ and the current version of the operator $\Gamma$.
In order to see how the two algorithms perform, we will next conduct some experiments both on synthetic and image data.

%%%%%%%%%%%%%%%%%%%%%%%%%%%%%%%
\subsection{Experiments on synthetic data}\label{sec:StupidNum}
%%%%%%%%%%%%%%%%%%%%%%%%%%%%%%%
In the first set of experiments, we use synthetic data generated from a given (target) analysis operator $\Omega$. A data vector $y$ is generated by choosing a vector $z$ from the unit sphere and a random subset $\Lambda$ of $\ell$ analysers. We then project $z$ onto the orthogonal complement of the chosen analysers, contaminate it with Gaussian noise and normalise it, see Table~\ref{tab:SignalModel}. The cosparse signals generated according to this model are very isotropic and thus do not exhibit the pathologies we described in the counterexample at the beginning of the section.\\
\begin{table}
	\centering
	\noindent\fbox{
		\parbox{0.465\textwidth}{
			{\bf Signal model($\Omega,\ell,\rho$)}\\
			Input:
			\begin{itemize}
				\item $\Omega\in\R^{K\times d}$ - target analysis Operator,
				\item $\ell$ - cosparsity level of the signals w.r.t. $\Omega$,
				\item $\rho$ - noise level.
			\end{itemize}
			Generation of the signals is done in the following way: 
			\begin{itemize}
				\item Draw $z\sim \Ncal(0,I_d), r\sim \Ncal(0,\rho^2 I_d)$ and $\Lambda\sim \Ucal( {[K]\choose\ell})$.
				\item Set 
				\begin{align}\label{eq:SignalModel}
				y = 
				\frac{(\id - \Omega_\Lambda^\dagger \Omega_\Lambda) z + r}{\Vert(\id - \Omega_\Lambda^\dagger \Omega_\Lambda) z + r\Vert}.
				\end{align}
			\end{itemize}
			The matrix $(\id -\Omega_\Lambda^\dagger \Omega_\Lambda)$ is a projector onto the space of all cosparse signals with cosupport $\Lambda$, so generating our signals in this way makes sure that they are (up to some noise) cosparse.
		}
	}
	\caption{Signal model}
	\label{tab:SignalModel}
\end{table}
\noindent {\bf Target operator:} As target operator for our experiments with synthetic data, we used a random operator of size $128\times 64$ consisting of rows drawn i.i.d. from the unit sphere $\S^{63}$.\\ 
{\bf Training signals:} Unless specified otherwise, in each iteration of the algorithm, we use $2^{17} = 131072$ signals drawn according to the signal model in Table~\ref{tab:SignalModel} with cosparsity level $\ell= 55$ and noiselevel $\rho = 0$ for noisefree resp. $\rho =0.2/\sqrt{d}$ for noisy data. We also conducted experiments with cosparsity level $\ell = 60$, but the results are virtually indistinguishable from the results for $\ell = 55$, so we chose not to present them here. We refer the interested reader to the AOL toolbox on the homepage of the authors\footnote{All experiments can be reproduced using the AOL Matlab toolbox available at \url{https://www.uibk.ac.at/mathematik/personal/schnass/code/aol.zip}.}, which can be used to reproduce the experiments.\\
{\bf Initialisation \& setup:} We use both a closeby and a random initialisation of the correct size. For the closeby initialisation, we mix the target operator 1:1 with a random operator and normalise the rows, that is, our initialisation operator is given by $\Gamma_0 = D_n (\Omega + R)$, where $R$ is a $K\times d$ matrix with rows drawn uniformly at random from the unit sphere $\S^{d-1}$ and $D_n$ is a diagonal matrix ensuring that the rows of $\Gamma_0$ are normalised. For the random initialisation we simply set $\Gamma_0 = R$. The correct cosparsity level $\ell$ is given to the algorithm and the results have been averaged over $5$ runs with different initialisations.\\
{\bf Recovery threshold:} We use the convention that an analyser $\omega_k$ is recovered if $\max_j|\ip{\omega_k}{\gamma_j}|\geq 0.99$.\\
Our first experiment is designed to determine the proportion of signals $L =\eps N$ that SAOL should use to estimate the values of $c_k$. We make an exploratory run for FAOL and SAOL with several choices of $\eps$, using 16384 noiseless, 60-cosparse signals per iteration and a random initialisation.
\ifthenelse{\boolean{onecol}}{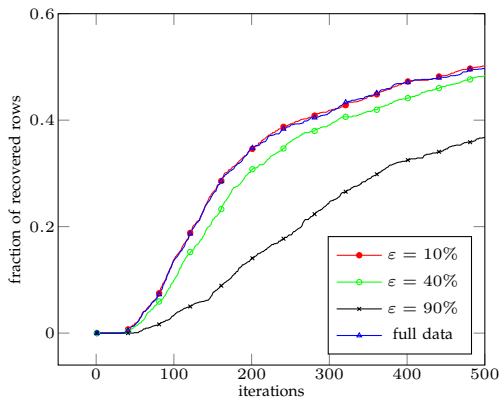
\begin{figure}[h!]
	\centering
	\begin{tikzpicture}
	\begin{axis}[
	xlabel=iterations,
	ylabel= fraction of recovered rows,
	yticklabel style = {font=\tiny,xshift=0.5ex},
	xticklabel style = {font=\tiny,yshift=0.5ex},
	ylabel style = {font=\tiny,yshift =-.65cm},
	xlabel style = {font=\tiny,yshift = .4cm},
	legend style = {font=\tiny},
	ymax=0.6 , 
	xmax = 500,
	legend pos=south east, 
	width = 0.4\columnwidth]	
	\addplot[color=red,mark = *, mark repeat = 40, mark size = 1pt] table[x=t,y=hits10]{NewPics/EpsilonEstimatorSequential.txt};  
	\addplot[color=green,mark = o, mark repeat = 40, mark size = 1pt] table[x=t,y=hits40]{NewPics/EpsilonEstimatorSequential.txt}; 
	\addplot[color=black,mark = x, mark repeat = 40, mark size = 1pt] table[x=t,y=hits90]{NewPics/EpsilonEstimatorSequential.txt}; 
	\addplot[color=blue,mark = triangle, mark repeat = 40, mark size = 1pt] table[x=t,y=hitsAll]{NewPics/EpsilonEstimatorSequential.txt}; 
	\legend{$\varepsilon = 10\%$,$\varepsilon = 40\%$ ,$\varepsilon = 90\%$, full data}
	
	\end{axis}
	\end{tikzpicture}
	
	\caption{Recovery rates of an exploratory run using FAOL and SAOL with different epsilons for the recovery of a random $128\times 64$ operator using 16384 samples in each iteration.}
	\label{fig:EpsChoice}
\end{figure}}
{\begin{figure}[h!]
		\centering
		\begin{tikzpicture}
		\begin{axis}[
		xlabel=iterations,
		ylabel= fraction of recovered rows,
		yticklabel style = {font=\tiny,xshift=0.5ex},
		xticklabel style = {font=\tiny,yshift=0.5ex},
		ylabel style = {font=\tiny,yshift =-.65cm},
		xlabel style = {font=\tiny,yshift = .4cm},
		legend style = {font=\tiny},
		ymax=0.6 , 
		xmax = 800,
		legend pos=south east, 
		width = 0.63\columnwidth]	
		\addplot[color=red,mark = *, mark repeat = 40, mark size = 1pt] table[x=t,y=hits10]{NewPics/EpsilonEstimatorSequential.txt};  
		\addplot[color=green, mark = o, mark repeat = 40, mark size = 1pt] table[x=t,y=hits40]{NewPics/EpsilonEstimatorSequential.txt}; 
		\addplot[color=black,mark = x, mark repeat = 40, mark size = 1pt] table[x=t,y=hits90]{NewPics/EpsilonEstimatorSequential.txt}; 
		\addplot[color=blue, mark = triangle, mark repeat = 40, mark size = 1pt] table[x=t,y=hitsAll]{NewPics/EpsilonEstimatorSequential.txt}; 
		\legend{$\varepsilon = 10\%$,$\varepsilon = 40\%$ ,$\varepsilon = 90\%$, full data}
		
		\end{axis}
		\end{tikzpicture}
		
		\caption{Recovery rates of an exploratory run using FAOL and SAOL with different epsilons for the recovery of a random $128\times 64$ operator using 16384 samples in each iteration.}
		\label{fig:EpsChoice}
\end{figure}}

The recovery rates in Figure~\ref{fig:EpsChoice} indicate that in order for the SAOL algorithm to achieve the best possible performance, $\varepsilon$ should be chosen small, meaning one should first get a good estimate of the gradients $g_k$. This allocation of resources also seems natural since for small $\varepsilon$ a large portion of the data is invested into estimating the $d$-dimensional vectors $g_k$, while only a small portion is used to subsequently estimate the numbers $c_k$. Based on these findings we from now on set $\eps = 10\%$ for the SAOL-algorithm.\\
In the next experiment we compare the recovery rates of FAOL, SAOL and Analysis-SimCO~\cite{dowadaplha16} from random and closeby initialisations in a noiseless setting.

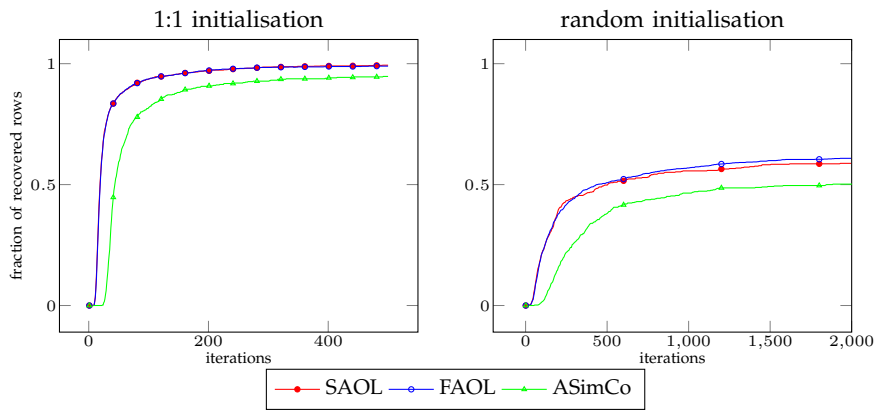
\begin{figure}
	\centering
	\begin{tikzpicture}
	\begin{groupplot}[
	group style={group size=2 by 1, horizontal sep = 1cm}]
	\nextgroupplot[title = {\footnotesize 1:1 initialisation}, title style={yshift=-1ex},legend to name={CommonLegend},legend style={legend columns=4, font = \scriptsize},
	xlabel=iterations,
	ylabel= fraction of recovered rows,
	ylabel style = {font=\tiny},
	yticklabel style = {font=\tiny,xshift=0.5ex},
	xticklabel style = {font=\tiny,yshift=0.5ex},
	ylabel style = {font=\tiny,yshift =-.65cm},
	xlabel style = {font=\tiny,yshift = .4cm},
	ymax=1.1 , 
	width = 0.35\columnwidth]
	
	\addplot[color=red, mark = *, mark size = 1pt, mark repeat = 40] table[x=t,y=l55n0]{NewPics/SAOL_Synthetic_Close.txt}; 
	\addplot[color=blue, mark = o, mark size = 1pt, mark repeat = 40] table[x=t,y=l55n0]{NewPics/FAOL_Synthetic_Close.txt}; 
	\addplot[color = green, mark = triangle, mark size = 1pt, mark repeat = 40] table[x=t, y=l55n0]{NewPics/ASimCO_Synthetic_Noisefree_Closeby.txt};
	
	\addlegendimage{red, mark=*}
	\addlegendentry{SAOL}
	\addlegendentry{FAOL}
	\addlegendentry{ASimCo}
	
	\nextgroupplot[title = {\footnotesize random initialisation},
	title style={yshift=-1ex},
	xlabel=iterations,
	yticklabel style = {font=\tiny,xshift=0.5ex},
	xticklabel style = {font=\tiny,yshift=0.5ex},
	xlabel style = {font=\tiny,yshift = .4cm},
	ymax=1.1 , 
	xmax = 2000,
	width = 0.35\columnwidth
	]
	\addplot[color=red, mark = *,mark size = 1pt, mark repeat = 120] table[x=t,y=l55n0]{NewPics/SAOL_Synthetic_Noisefree.txt};  
	\addplot[color=blue, mark = o,mark size = 1pt, mark repeat = 120] table[x=t,y=l55n0]{NewPics/FAOL_Synthetic_Noisefree.txt};
	\addplot[color = green, mark = triangle, mark size = 1pt, mark repeat = 120] table[x=t, y=l55n0]{NewPics/ASimCO_Synthetic_Noisefree.txt};
	
	\end{groupplot}
	\path (group c1r1.south east) -- node[below=10pt]{\ref{CommonLegend}} (group c2r1.south west);
	\end{tikzpicture}
	\caption{Recovery rates of SAOL, FAOL and ASimCo from signals with various cosparsity levels $\ell$ in a noiseless setting, using closeby (left) and random (right) initialisations for cosparsity level $\ell =55$.}
	\label{fig:StupidNoisefree}
\end{figure}
\begin{figure}
	\centering
	\includegraphics[width=0.3\textwidth, height = 0.6\textwidth]{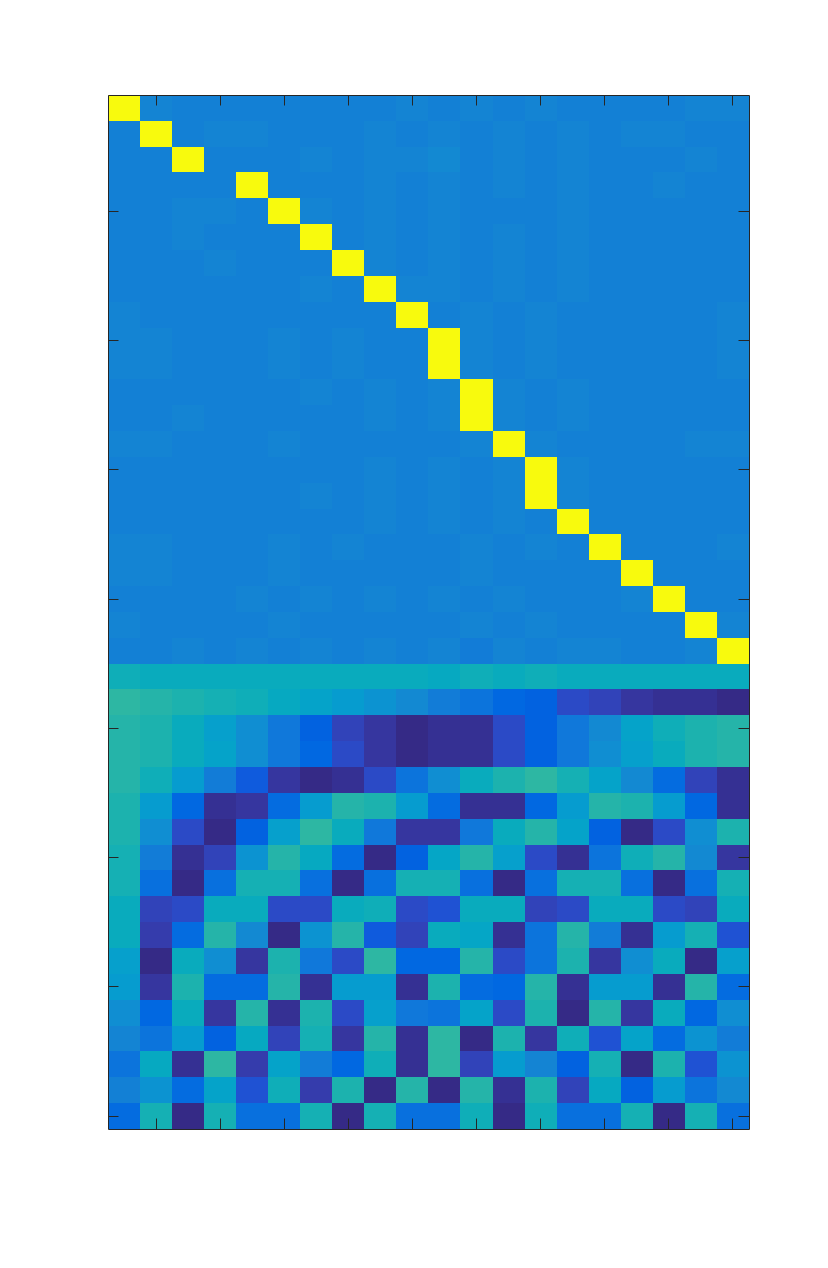}
	\quad
	\includegraphics[width=0.3\textwidth, height = 0.6\textwidth]{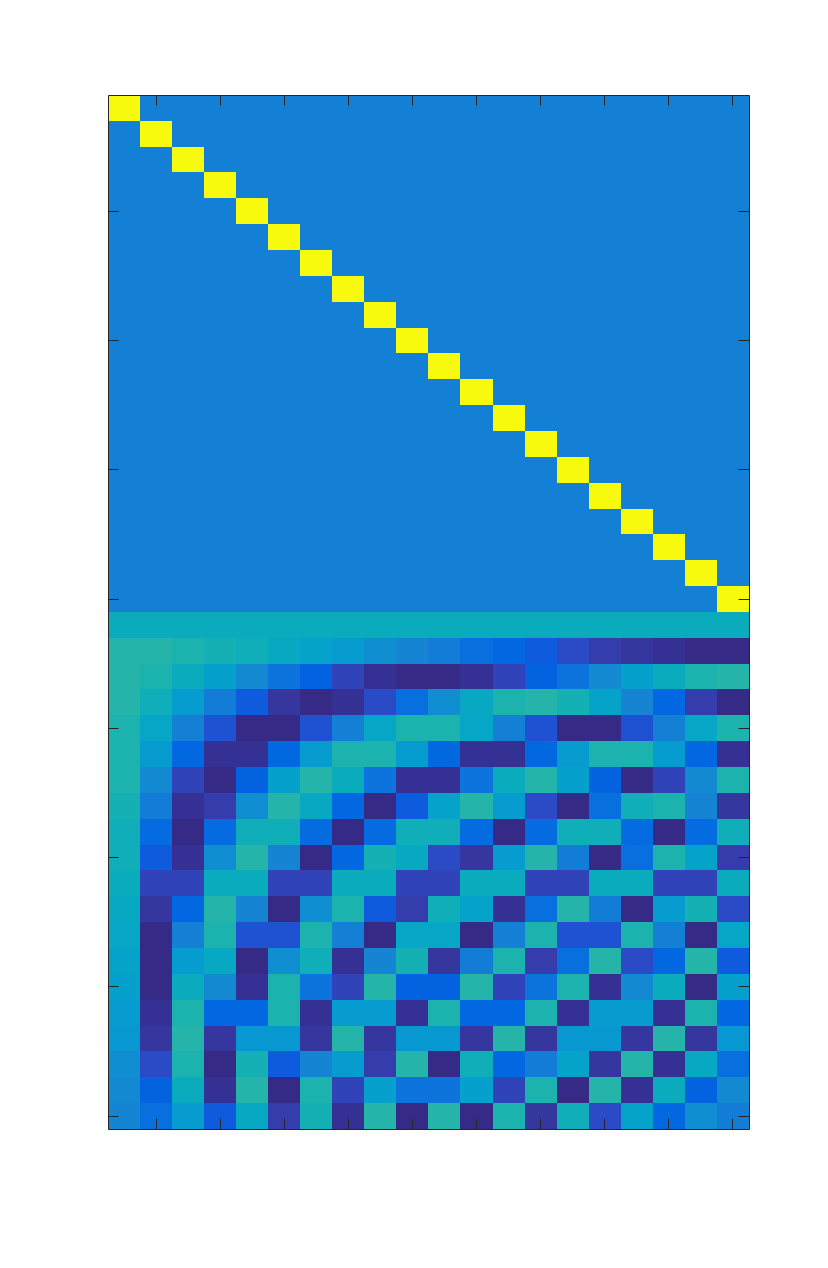}
	\caption{Operator learned with FAOL from a random initialisation (left) vs the original Dirac-DCT operator (right). The rows of the learned operator have been reordered and the signs have been matched with the original operator for easier comparison. For the learning $300$ iterations with $8192$ noiseless 12-cosparse signals, constructed according to the model in Table~\ref{tab:SignalModel} were used.}
	\label{fig:MultipleRecovery}
\end{figure}
The first good news of the results, plotted in Figure~\ref{fig:StupidNoisefree}, is that the sequential algorithm SAOL with estimated stepsize performs as well as the one with explicitly calculated optimal stepsize. We also see
that with a closeby initialisation both algorithms recover the target operator (almost) perfectly for both cosparsity levels, which indicates that locally our algorithms perform as expected.
With a random initialisation the algorithms tend to saturate well below full recovery. This is not surprising, as the nonconvex optimisation we perform depends heavily on the initialisation. In case of the closeby initialisation, we set each row of the starting operator near the desired row of the target operator. In contrast, for the random initialisation it is very likely that two rows of the initialised operator lie close to the same row of the target operator. Our algorithms then tend to find the nearest row of the target operator and thus we get multiple recovery of the same row. As we have prescribed a fixed number of rows, another row must be left out, which leads to the observed stagnation of the recovery rates and means that we are trapped in a local minimum of our target function. Figure~\ref{fig:MultipleRecovery} illustrates this effect for the Dirac-DCT operator in $\R^{40\times 20}$. \\

Since the phenomenon of recovering duplicates is not only as old as analysis operator learning but as old as dictionary learning, \cite{ahelbr06}, there is also a known solution to the problem, which is the replacement of coherent analysers or atoms. 

\subsection{Replacement}\label{sec:ReplacementSimple}
A straightforward way to avoid learning analysis operators with duplicate rows is to check after each iteration, whether two analysers of the current iterate $\Gamma$ are very coherent. Under the assumption that the coherence of the target operator $\mu(\Omega) = \max_{i\neq j\in [K]} |\ip{\omega_i}{\omega_j}|$ is smaller than some threshold $\mu(\Omega)\leq \mu_0$, we know that two rows of $\gamma_i,\gamma_j$ are likely to converge to the same target analyser, whenever we have $|\ip{\gamma_i}{\gamma_j}|> \mu_0$.\\
In this case, we perform the following decorrelation procedure. During the algorithm, we monitor the activation of the individual rows of the operator, that is, if the $k-$th row of the operator is used for a set $J_n$, we increment a counter $v_k$. If now two rows $\gamma_i$ and $\gamma_j$ have overlap larger than $\mu_0$, we compare the numbers $v_i$ and $v_j$ and keep the row with larger counter. Without loss of generality suppose $v_i>v_j$. We then subtract the component in direction of $\gamma_i$ from $\gamma_j$, namely $\tilde \gamma_j = \gamma_j - \ip{\gamma_i}{\gamma_j}\gamma_i$ and renormalise.\\
This decorrelation is different from the one that is performed in~\cite{dowadaplha16}, but has the merit that correct rows that have already been found do not get discarded or perturbed. This is especially useful in the case we consider most likely, where one row already has large overlap with a row of the target operator and another row slowly converges towards the same row. Then our decorrelation procedure simply subtracts the component pointing in this direction.\\
Since, unlike dictionaries, analysis operators can be quite coherent and still perform very well, for real data it is recommendable to be conservative and set the coherence threshold $\mu_0$ rather high.

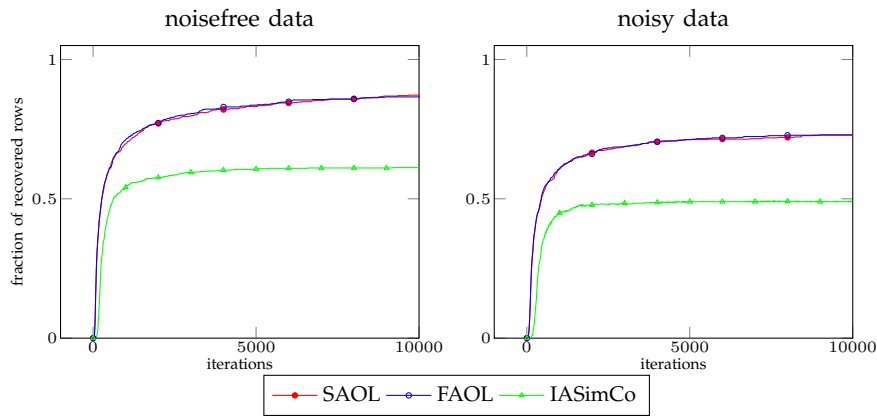
\begin{figure}
	\centering
	\begin{tikzpicture}
	\begin{groupplot}[
	group style={group size=2 by 1, horizontal sep = 1cm}]
	\nextgroupplot[title = {\footnotesize noisefree data}, title style={yshift=-.5ex},legend to name={CommonLegend},legend style={legend columns=4, font = \scriptsize},
	xlabel=iterations,
	ylabel= fraction of recovered rows,
	yticklabel style = {font=\tiny,xshift=0.5ex},
	xticklabel style = {font=\tiny,yshift=0.5ex},
	ylabel style = {font=\tiny,yshift =-.65cm},
	xlabel style = {font=\tiny,yshift = .4cm},
	ymax=1.05 , 
	ymin = 0,
	width = 0.35\columnwidth,
	xtick = {0,5000,10000},
	xmax = 10000,
	xticklabels = {$0$,$5000$,$10000$}]
	\addplot[color=red, mark = *, mark size = 1pt, mark repeat =100] table[x=t,y=l55n0]{NewPics/SAOL_Synthetic_Replace.txt};  
	\addplot[color=blue, mark = o, mark size = 1pt, mark repeat = 100] table[x=t,y=l55n0]{NewPics/FAOL_Synthetic_Replace.txt}; 
	green\addplot[color=green, mark = triangle, mark size = 1pt, mark repeat = 100] table[x=t,y=targSimCo]{NewPics/CompareAllSyntheticNoisefreeReplacement.txt}; 
	\addlegendimage{red, mark=*}
	\addlegendentry{SAOL}
	\addlegendentry{FAOL}
	\addlegendentry{IASimCo}
	
	\nextgroupplot[title = {\footnotesize noisy data},
	title style={yshift=-1ex},
	xlabel=iterations,
	yticklabel style = {font=\tiny,xshift=0.5ex},
	xticklabel style = {font=\tiny,yshift=0.5ex},
	ylabel style = {font=\tiny,yshift =-.65cm},
	xlabel style = {font=\tiny,yshift = .4cm},
	ymax=1.05 ,
	ymin = 0, 
	width = 0.35\columnwidth,
	xtick = {0,5000,10000},
	xmax = 10000,
	xticklabels = {$0$,$5000$,$10000$}
	]
	\addplot[color=red, mark = *, mark size = 1pt, mark repeat =100] table[x=t,y=l55n02]{NewPics/SAOL_Synthetic_Replace.txt}; 
	\addplot[color=blue, mark = o, mark size = 1pt, mark repeat =100] table[x=t,y=l55n02]{NewPics/FAOL_Synthetic_Replace.txt};
	\addplot[color=green, mark = triangle, mark size = 1pt, mark repeat =100] table[x=t,y=targSimCo]{NewPics/CompareAllSyntheticNoisyReplacement.txt};
	
	\end{groupplot}
	\path (group c1r1.south east) -- node[below=10pt]{\ref{CommonLegend}} (group c2r1.south west);
	\end{tikzpicture}
	\caption{Recovery rates of SAOL and FAOL with replacement from signals with cosparsity level $\ell=55$ in a noiseless (left) and a noisy setting (right), using a random initialisation.}	
	\label{fig:ReplacementSimple}
\end{figure}

Figure~\ref{fig:ReplacementSimple} shows the recovery results of our algorithm with the added replacement step for $\mu_0=0.8$, when using a random initialisation and the same settings as described at the beginning of the section. \\
We see that in the noiseless case, after 10000 iterations almost $90\%$ of the signals have been recovered. If we introduce a small amount of noise, however, significantly fewer rows are recovered. To avoid repetition we postpone a thorough comparison of FAOL/SAOL to ASimCo on synthetic data to Section~\ref{sec:comparison} after the introduction of our other two algorithms in Section~\ref{sec:backward}, however we can already observe now that both SAOL and FAOL perform better than  IASimCO. 
Next we take a look at how the optimal stepsize affects learning on image data.
%

%%%%%%%%%%%%%%%%%%%%%%%%%%%%%%%%
\subsection{Experiments on image data}\label{sec:StupidRealData}
%%%%%%%%%%%%%%%%%%%%%%%%%%%%%%%%%
To get an indication how our algorithms perform on real data, we will them to learn a quadratic analysis operator on all $8\times 8$ patches of the $256\times 256$ Shepp Logan phantom, cf. Figure~\ref{fig:TrainingImages}. We initialise the analysis operator $\Gamma\in\R^{64\times 64}$ randomly as for the synthetic data and set the cosparsity level $\ell= 57$, the parameter $\eps=10\%$ and the replacement threshold $\mu_0=0.99$. For each iteration we choose 16384 out of the available 62001 patches uniformly at random as training signals. Since we do not have a reference operator for comparison this time, we compare the value of target function after each iteration, as plotted in Figure~\ref{fig:SAOLFAOLRealData}.
\begin{figure*}[!hb]
	\begin{center}
		\includegraphics[width=0.2\textwidth]{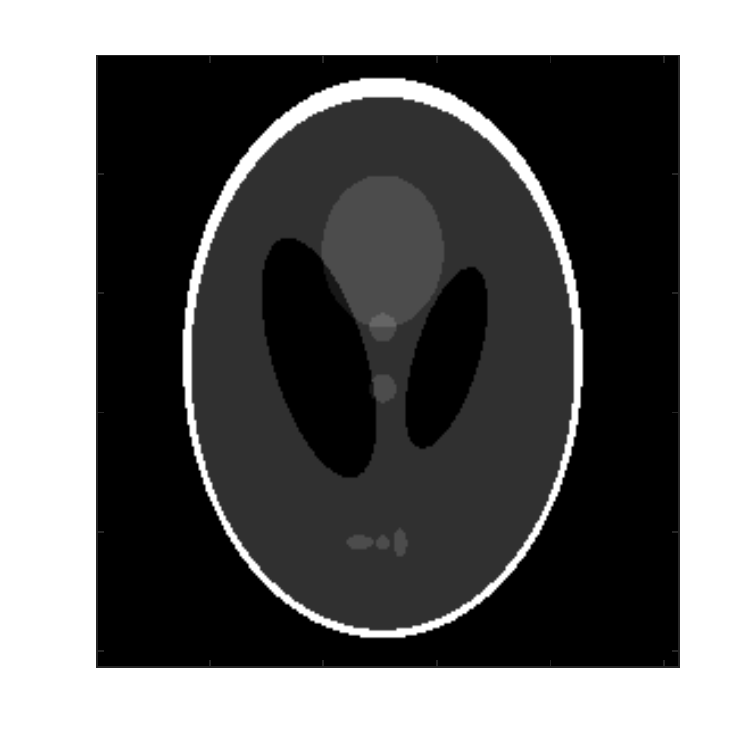}
		\quad
		\begin{tikzpicture}
		\begin{axis}[xlabel=iterations,
		yticklabel style = {font=\tiny,xshift=0.5ex},
		xticklabel style = {font=\tiny,yshift=0.5ex},
		ylabel style = {font=\tiny,yshift =-.5cm},
		xlabel style = {font=\tiny,yshift = .4cm},
		legend style = {font=\small,nodes={scale=0.5, transform shape}},
		legend pos=north east, 
		width = 0.25\textwidth, 
		height = 0.23\textwidth,
		ymode = log, ylabel= value of target function,
		ytick = {0.015625,0.0078125,0.00390625},
		yticklabels = {$2^{-6}$,$2^{-7}$,$2^{-8}$},
		xtick = {0,7500,15000},
		xticklabels = {$0$,$7500$,$15000$},
		xmax = 19950,
		ymin = 0.003]
		\addplot[color=red,mark = *,mark repeat = 200, mark size = 1pt] table[x=t,y=targSAOL]{NewPics/DecayTargetSheppLoganSquareManyIter.txt};
		\addplot[color=blue, mark = o, mark repeat =200, mark size = 1pt] table[x=t,y=targFAOL]{NewPics/DecayTargetSheppLoganSquareManyIter.txt};   	 
		\addplot[color=green, mark = triangle, mark repeat = 100, mark size = 1pt] table[x=t,y=targ]{NewPics/DecayTargetSheppLoganSquareManyIterSimCo.txt};  	 
		\legend{SAOL,FAOL, ASimCo}
		\end{axis}
		\end{tikzpicture}
		\qquad
		\includegraphics[width=0.2\textwidth, height=0.2\textwidth]{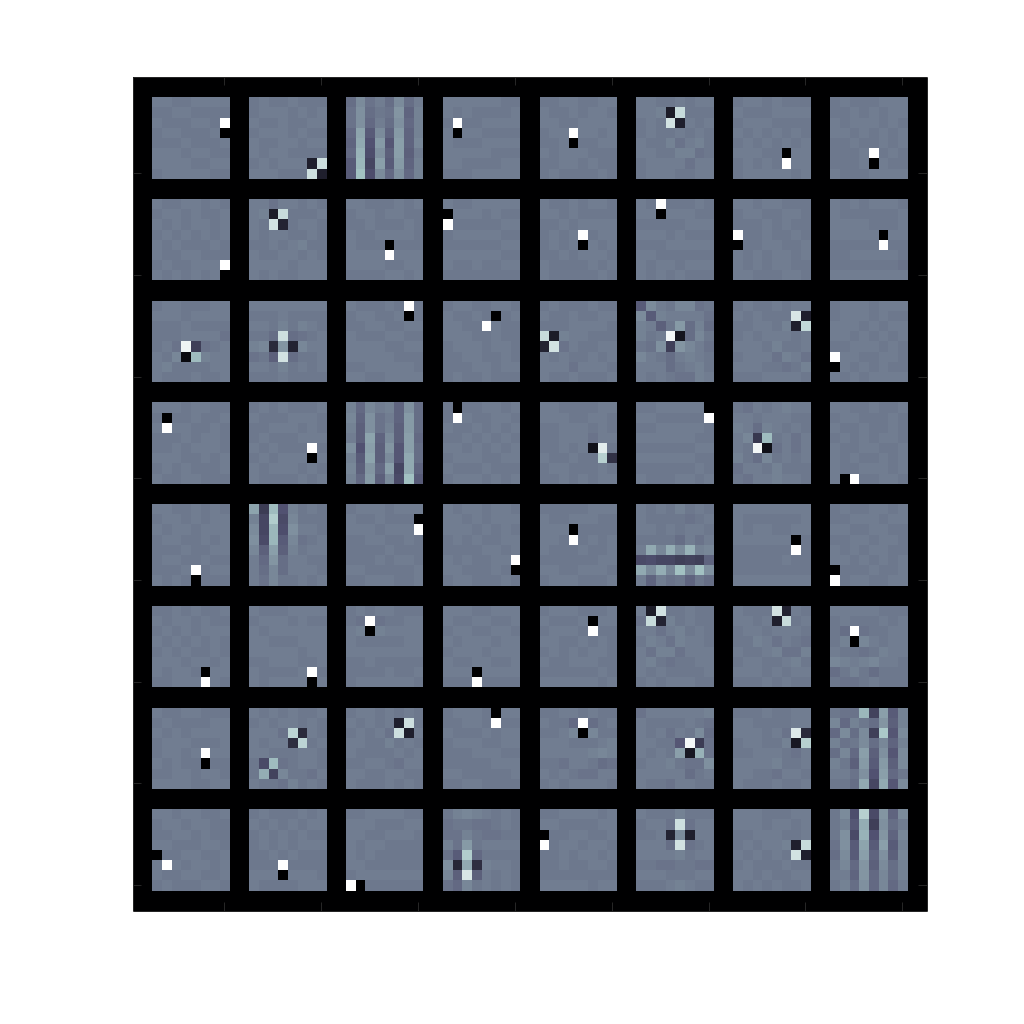}
		\includegraphics[width=0.2\textwidth, height=0.2\textwidth]{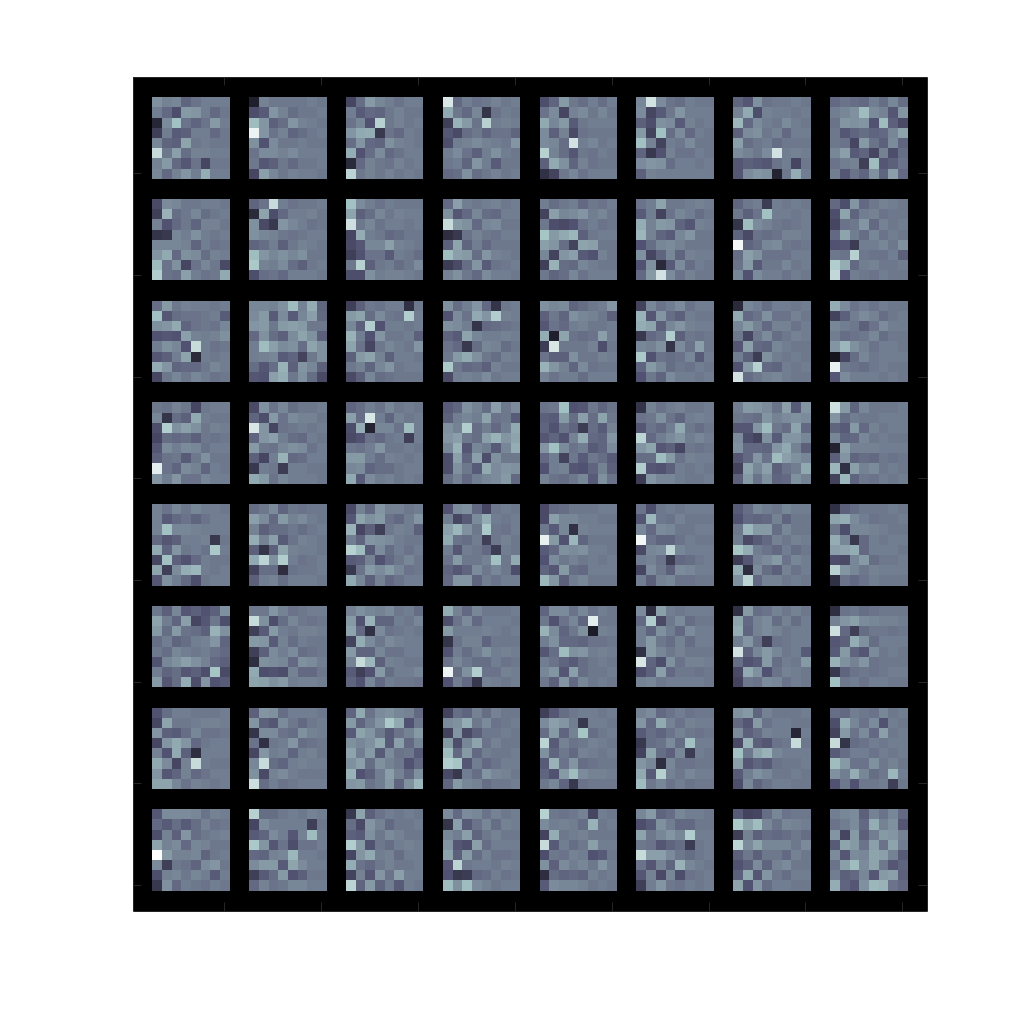}
		\caption{Shepp Logan phantom (1st), value of the target function for both algorithms (2nd), operator obtained by FAOL (3rd) and by ASimCo (4th) after 100000(!) iterations.}
		\label{fig:SAOLFAOLRealData}
	\end{center}
\end{figure*} 
We can see that the target function is only decreased very slowly by all algorithms, where ASimCO saturates at a seemingly suboptimal value, which is also illustrated by the recovered operator shown in Figure~\ref{fig:SAOLFAOLRealData}.

As we choose the optimal stepsize for FAOL, we further cannot hope to increase convergence speed significantly using an explicit descent algorithm.

Still, if we look at the learned operator, we can see the merit of our method. After 100000 iterations, the learned operator seems to consist of pooled edge detectors, which are known to cosparsify piecewise constant grayscale images. Note also that the $d\times d$ analysis operator is naturally very different from any $d\times d$ dictionary we could have learned with corresponding sparsity level $S=d-\ell$, see e.g \cite{sc15}. This is due to the fact that image patches are not isotropic, but have their energy concentrated in the low frequency ranges. So while both the $d\times d$ dictionary and the analysis operator will not have (stable) full rank, the dictionary atoms will tend to be in the low frequency ranges, and the analysers will - as can be seen - tend to be in the high frequency ranges.\\
We also want to mention that for image data the replacement strategy for $\mu_0=0.99$ is hardly ever activated. Lowering the threshold results in continuous replacement and refinding of the same analysers. This phenomenon is again explained by the lack of isotropy and the shift invariant structure of the patch data, for which translated and thus coherent edge detectors, as seen in Figure~\ref{fig:SAOLFAOLRealData}, naturally provide good cosparsity.\\
Encouraged by the learned operator we will explore in the next section how to stabilise the algorithm and accelerate its convergence on image data.

%%%%%%%%%%%%%%
\section{Two implicit operator learning algorithms - IAOL and SVAOL}\label{sec:backward}
%%%%%%%%%%%%%%%%
Due to the small optimal stepsize that has to be chosen on real data and the resulting slow convergence, we need to rethink our approach and enforce stability of the algorithm even with larger stepsizes.  

\subsection{The IAOL algorithm}

In standard gradient descent, for each row of $\Gamma$, we have the iteration 
\begin{align}
\bar \gamma_k = \gamma_k -\alpha\nabla f_N(\Gamma)_k.
\end{align}
Rewriting yields \begin{align}\frac{\bar \gamma_k - \gamma_k}{\alpha}=-\nabla f_N(\Gamma)_k,\end{align}which can be interpreted as an explicit Euler step for the system of ordinary differential equations
\begin{align}\label{eq:EmpiricalODE}
\dot \gamma_k =  -\nabla f_N(\Gamma)_k,\,\, k\in[K].
\end{align}
The explicit Euler scheme is the simplest integration scheme for ordinary differential equations and known to have a very limited region of convergence with respect to the stepsize. In our case, this means that we have to choose extremely small values for the descent parameter $\alpha$ in order to achieve convergence.\\
The tried and tested strategy to overcome stability issues when numerically solving differential equations is to use an implicit scheme for the integration,~\cite{hairer1993Solving,hairer2010Solving}.
We will use this as an inspiration to obtain a more stable learning algorithm.\\
We briefly sketch the ideas behind an implicit integration scheme. Suppose we want to solve the differential equation $\dot x = f(x)$. If we discretise $x(t)$ and approximate the derivative by $\dot x(t_n) \approx \tfrac{x(t_n)-x(t_{n-1})}{t_n-t_{n-1}}$, we have to choose whether we use the approximation $\dot x(t_n) = f(x(t_n))$ or $\dot x(t_n) = f(x(t_{n-1}))$.
Choosing $f(x(t_{n-1}))$ yields the explicit Euler scheme, which in our setting corresponds to the FAOL algorithm. If we choose $f(x(t_{n}))$ we obtain the implicit Euler scheme and need to solve 
\begin{align}
\frac{x(t_n)-x(t_{n-1})}{t_n-t_{n-1}} = f(x(t_n)).
\end{align}
If $f(x) = Ax$ is linear, this leads to the recursion 
\begin{align}
x(t_n) = (\id -(t_n-t_{n-1})A)^{-1} x(t_{n-1}),
\end{align}
and in each step we need to solve a system of linear equations. This makes implicit integration schemes inherently more expensive than explicit schemes. However, in return we get additional stability with respect to the possible stepsizes. If $f$ is a nonlinear function, the inversion is more difficult and can often only be approximated for example via a Newton method.\\
Mapping everything to our setting, we observe that the gradient $\nabla f_N(\Gamma)$ is nonlinear because the sets $J_n$ depend on $\Gamma$. Still, due to the special structure of the gradient $\nabla f_N(\Gamma)$, it has a simple linearisation, $\nabla f_N(\Gamma)_k = 2\gamma_k \sum_{n\colon k\in J_n}y_n y_n^\star$.
We can now use the current iterate of $\Gamma$ to compute the matrix $A_k(\Gamma):=\sum_{n\colon k\in J_n}y_n y_n^\star$ and to linearise the equation. For our operator learning problem, we get the following linearised variant of the implicit Euler scheme
\begin{align}
\frac{\bar \gamma_k-\gamma_k}{\alpha} = -\bar \gamma_k A_k(\Gamma),
\end{align}
leading to the recursion
\begin{align}
\bar \gamma_k = \gamma_k (\id + \alpha A_k(\Gamma))^{-1} 
\end{align}
Due to the unconditional stability of the implicit Euler scheme,~\cite{hairer2010Solving}, we can take $\alpha$ considerably larger than in case of FAOL or SAOL. We only need to make sure that one step of the algorithm does not take us too close to zero, which is a stable attractor of the unconstrained system. In order to stay within the manifold $\Acal$, we again have to renormalise after each step. The final algorithm is summarised in Table~\ref{tab:IAOL}.\\
\begin{table}
	\centering
	\noindent\fbox{%
		\parbox{0.465\textwidth}{%
			{\bf IAOL($\Gamma,\ell,Y,\alpha$) - (one iteration)} \\
			\begin{itemize}
				\item For all n:
				\begin{itemize}
					\item Find $J_n = \argmin_{|J| = \ell} \Vert \Gamma_J y_n\Vert_2^2$.
					\item For all $k\in J_n$ update $A_k=A_k+y_n y_n^\star$.
				\end{itemize}
				\item For all $k\in [K]$ set $\bgamma_k = \gamma_k \left(\id  + \alpha A_k\right)^{-1}$.
				\item Output $\bGamma =(\frac{\bgamma_1}{\Vert \bgamma_1\Vert_2},\ldots,\frac{\bgamma_K}{\Vert \bgamma_K\Vert_2})^\star$.
			\end{itemize}
		}
	}
	\caption{The IAOL algorithm}
	\label{tab:IAOL}
\end{table}

Let us take a short look at the computational complexity of the implicit algorithm and the price we have to pay for increased stability. As in the previous section, we need to compute all products of the vectors $y_n$ with the current iterate $\Gamma$, costing $\Ocal(NKd)$. Further, in each step we need to solve $K$ linear systems of size $d\times d$, amounting to an additional cost of $\Ocal(Kd^2)$. So, altogether for one step, we arrive at $\Ocal(NKd+Kd^2)=\Ocal(NKd)$. 
However, in contrast to FAOL, the IAOL algorithm cannot be made sequential, unless we are willing to store the $K$ matrices $A_k$ in each step. This amounts to an additional spatial complexity of $\Ocal(Kd^2)$, and only pays off for $N>Kd$, since the storage cost of the data matrix is $\Ocal(Nd)$. Note that in contrast to FAOL, the explicit calculation of the matrices $A_k$ is necessary, since we need to solve a system of linear equations.
As for the FAOL algorithm, we can guarantee decrease or preservation of the target function by the IAOL algorithm.

\begin{theorem}\label{thm:IAOLdecrease}
	The IAOL algorithm decreases or preserves the value of the target function in each step regardless of the choice of $\alpha>0$. Preservation rather than decrease of the objective function can only occur if all rows $\gamma_k$ of the current iterate $\Gamma$ are eigenvectors of the matrices $A_k(\Gamma)$.
\end{theorem} 

\begin{proof} In order to simplify notation, we drop the indices and write $\Gamma$ for the current iterate. As we will do the computations only for a fixed row $\gamma_k$ of $\Gamma$, we denote it by $\gamma$. We further write $A$ for the matrix corresponding to $\gamma$ and $\bgamma$ for the next iterate.
We want to show
\[\bgamma A \bgamma^\star - \gamma A\gamma^\star = \frac{\gamma A(\id +\alpha A)^{-2} \gamma^\star}{\gamma (\id +\alpha A)^{-2}\gamma^\star} - \gamma A\gamma^\star\leq0,\]
which is implied by
\ifthenelse{\boolean{onecol}}{
\begin{align*}\gamma(A(\id +\alpha A)^{-2} - A\gamma^\star\gamma (\id +\alpha A)^{-2})\gamma^\star = \gamma (A(\id - \gamma^\star\gamma)(\id +\alpha A)^{-2})\gamma^\star \leq0,\end{align*}}
{\begin{align*}\gamma(A(\id +\alpha A)^{-2} - A\gamma^\star\gamma (\id +\alpha A)^{-2})\gamma^\star =\\= \gamma (A(\id - \gamma^\star\gamma)(\id +\alpha A)^{-2})\gamma^\star \leq0,\end{align*}}
as the denominator is positive. We now use the eigendecomposition of the symmetric, positive semidefinite matrix $A$, that is $A = \sum_i \lambda_i u_i u_i^\star$, where $\lambda_i\geq0$ for all $i$ and $(u_i)_{i\in[d]}$ is an orthonormal basis.

Inserting this, a short computation shows that
\ifthenelse{\boolean{onecol}}
{\begin{align*}
\gamma (A(\id - \gamma^\star\gamma)(\id +\alpha A)^{-2})\gamma^\star = \sum_i \sum_{l\neq i} \lambda_i\underbrace{\frac{\alpha (2+\alpha(\lambda_i+\lambda_l))|\ip{\gamma}{u_i}|^2|\ip{\gamma}{u_l}|^2}{(1+\alpha \lambda_i)^2(1+\alpha\lambda_l)^2}}_{=:a_{il}}(\lambda_l-\lambda_i).
\end{align*}}
{
\begin{multline*}
	\gamma (A(\id - \gamma^\star\gamma)(\id +\alpha A)^{-2})\gamma^\star =\\= \sum_i \sum_{l\neq i} \lambda_i\underbrace{\tfrac{\alpha (2+\alpha(\lambda_i+\lambda_l))|\ip{\gamma}{u_i}|^2|\ip{\gamma}{u_l}|^2}{(1+\alpha \lambda_i)^2(1+\alpha\lambda_l)^2}}_{=:a_{il}}(\lambda_l-\lambda_i).
	\end{multline*}}
Note that $a_{il} = a_{li}\geq0$. Further, we can drop the condition $l\neq i$ in the sums above, as the term corresponding to the case $i=l$ is zero.

In order to show that the sum $S_1:= \sum_{i,l} \lambda_i a_{il}(\lambda_l - \lambda_i)$ is never positive, define a second sum $S_2 := \sum_{i,l} \lambda_l a_{il}(\lambda_l - \lambda_i)$. Then, by antisymmetry, we have that $S_1+S_2 = \sum_{i,l} (\lambda_i+\lambda_l) a_{il}(\lambda_l - \lambda_i) = 0$. 
Further, $S_2 -S_1= \sum_{i,l} a_{il}(\lambda_l - \lambda_i)^2\geq0$, from which follows that $S_1\leq 0$. The whole discussion is independent of $\alpha>0$, so any viable choice of $\alpha$ results in a decrease of the objective function.

Assume now that $\bgamma A \bgamma^\star - \gamma A\gamma^\star = 0$. Then also $\gamma (A(\id - \gamma^\star\gamma)(\id +\alpha A)^{-2})\gamma^\star = S_1 =0$. This in turn implies that $S_2 = 0$ and $S_2-S_1 = 0$.
As every term in $\sum_{i,l} a_{il}(\lambda_l - \lambda_i)^2$ is positive or zero, we have that for all $i\neq l$ also $a_{il}(\lambda_l - \lambda_i)^2$ must be zero.
If all eigenvalues of the matrix $A$ are distinct this implies that $a_{il} = 0$ for all $i\neq l$. This in turn implies that for all $i\neq l$ the product $|\ip{\gamma}{u_i}|^2|\ip{\gamma}{u_l}|^2=0$, so either the overlap of $\gamma_k$ with $u_i$ or $u_l$ is zero. But this means that $\gamma_k$ must be equal to one of the eigenvectors.
If not all eigenvalues of the matrix $A$ are distinct, then the previous discussion still holds for the eigenvalues which are distinct. Assume that $i\neq l$ and $\lambda_i = \lambda_j$. Then $a_{il}(\lambda_l - \lambda_i)^2 = 0$ regardless of the value of $a_{il}$, so if $\gamma\in \mathrm{span}\{u_i,u_l\}$, we still have that $S_2-S_1 = 0$. This shows that in all cases, where the target function does not decrease, $\gamma$ needs to be an eigenvector of $A$.
\end{proof}

Note that the IAOL algorithm essentially performs a single step of an inverse iteration to compute the eigenvectors corresponding to the smallest eigenvalues of the matrices $A_k$. We will use this fact in the next section to introduce our last algorithm to learn analysis operators.

\subsection{The SVAOL algorithm}

Revisiting condition~\eqref{eq:alphaCondition} suggests another algorithm for learning analysis operators. The stepsize choice essentially amounts to
\[\bgamma_k = \argmin_{v\in\mathcal{K}_2(\gamma_k,A_k)\cap(\S^{d-1})^\star}\frac{v A_k v^\star}{v v^\star},\] 
where $\mathcal{K}_2(\gamma_k,A_k) = \text{span}\{\gamma_k,\gamma_k A_k\}$, the Krylov space of order 2.
Removing the restriction that $\bgamma_k$ must lie in the Krylov space $\mathcal{K}_2(\gamma_k,A_k)$ yields the update step
\[\bgamma_k = \argmin_{v\in(\S^{d-1})^\star}v A_k v^\star,\]
which means that $\bgamma_k$ is the eigenvector corresponding to the smallest eigenvalue of the matrix $A_k$. The resulting algorithm, called SVAOL, is summarised in Table~\ref{tab:SVAOL}.

\begin{table}
	\centering
	\noindent\fbox{%
		\parbox{0.465\textwidth}{%
			{\bf SVAOL($\Gamma,\ell,Y$) - (one iteration)} \\				For all $k\in[K]$, set $A_k = 0$.
			\begin{itemize}
				
				\item For all n:
				\begin{itemize}
					\item Find $J_n = \argmin_{|J| = \ell} \Vert \Gamma_J y_n\Vert_2^2$.
					\item For all $k\in[K]$ update $A_k = A_k+y_n y_n^\star$ if $k\in J_n$.
				\end{itemize}
				\item For all $k\in[K]$, set $\bar \gamma_k = \argmin_{v\in(\S^{d-1})^\star}v A_k v^\star$.
				\item Output $\bGamma =(\bgamma_1,\ldots,\bgamma_K)^\star$.
			\end{itemize}
		}
	}
	\caption{The SVAOL algorithm}
	\label{tab:SVAOL}
\end{table}

The obtained SVAOL algorithm bears close resemblance to the 'Sequential Minimal Eigenvalues' algorithm devised in~\cite{ophir2011sequential}. However, a key difference is that the computation of the rows of the target operator is not done sequentially in the SVAOL algorithm. Furthermore, the following theorem concerning decrease of the target function can be established.

\begin{theorem}\label{thm:SVAOLdecrease} 
The SVAOL algorithm decreases or preserves the value of the target function in each step. The only case when it preserves the value of the target function is when the rows $\gamma_k$ are already eigenvectors corresponding to the smallest eigenvalues of the matrices $A_k$.
\end{theorem}

The results for SVAOL given in Theorem~\ref{thm:SVAOLdecrease} improve the results obtained for IAOL in Theorem~\ref{thm:IAOLdecrease}. 
Now the decrease of the target function can be guaranteed if not all rows of the current iterate $\Gamma$ are already the eigenvectors corresponding to the smallest eigenvalues of the matrices $A_k$.\footnote{This means that if the target function is differentiable in $\Gamma$ and cannot be decreased by the SVAOL algorithm, we have already arrived at a local minimum. As we have stated previously, however, the target is not differentiable everywhere and thus this cannot be used to derive a local optimality result.}

\begin{proof}
 To show this, denote by $(A_k)_{k\in [K]}$ and $(\bar A_k)_{k\in [K]}$ the matrices defined in Table~\ref{tab:SVAOL} for the operators $\Gamma$ and $\bGamma$, respectively. Further denote by $(\sigma_k)_{k\in[K]}$ and $(\bar\sigma_k)_{k\in[K]}$ their smallest singular values.

Then

\begin{align*}
	f(\Gamma) & = \sum_{k=1}^K \gamma_k A_k \gamma_k^\star \geq \sum_{k=1}^K \sigma_k = \sum_{k=1}^K \bgamma_k A_k\bgamma_k^\star =\\
	&= \sum_{k=1}^K \sum_{n\colon k\in J_n} |\ip{\bgamma_k}{y_n}|^2 \geq \sum_{k=1}^K \sum_{n\colon k\in \bar J_n} |\ip{\bgamma_k}{y_n}|^2 =\\
	&= \sum_{k=1}^K \bgamma_k \bar A_k \bgamma_k^\star = f(\bGamma)
	\end{align*}
due to the definition of the sets $J_n$ and $\bar J_n$.
The first inequality is strict, except in the case when $\gamma_k$ are already eigenvectors of $A_k$ corresponding to the smallest eigenvalues.
\end{proof}

Finding the eigenvectors corresponding to the smallest eigenvalues of the matrices $A_k$ is indeed the desired outcome, which can be seen as follows.
First, note that the matrices $A_k(\Gamma)$ are (up to a constant) empirical estimators of the matrices $\mathbb{A}_k(\Gamma):=\E y y^\star \chi_{\{y\colon k\in J^\Gamma(y)\}}$, where $J^\Gamma(y) = \argmin_{|J| = \ell} \Vert \Gamma_J y\Vert_2^2$.
The rows $\omega_k$ of the target operator $\Omega$ are (in the noisefree setting) always eigenvectors to the eigenvalue zero for the matrix $\mathbb{A}_k(\Omega)$, since according to the signal model given in Table~\ref{tab:SignalModel}, we have
\ifthenelse{\boolean{onecol}}{
\begin{eqnarray*}
\mathbb{A}_k(\Omega) = \E y y^\star\chi_{\{y\colon k\in J^\Omega(y)\}}&=& \E (\id - \Omega^\dagger_\Lambda \Omega_\Lambda) z z^\star(\id - \Omega^\dagger_\Lambda \Omega_\Lambda)\chi_{\{(\Lambda,z)\colon k\in\argmin_{|J| = \ell} \Vert \Omega_J (\id - \Omega^\dagger_\Lambda \Omega_\Lambda) z\Vert_2^2\}}\\
&=&\E_\Lambda (\id - \Omega^\dagger_\Lambda \Omega_\Lambda)\E_z z z^\star(\id - \Omega^\dagger_\Lambda \Omega_\Lambda)\chi_{\{\Lambda\colon k\in\Lambda\}}\\
&=&\E_\Lambda  (\id - \Omega^\dagger_\Lambda \Omega_\Lambda)\chi_{\{\Lambda\colon k\in\Lambda\}}\\
&=&{K\choose\ell}^{-1}\sum_{\Lambda\colon k\in\Lambda}  (\id - \Omega^\dagger_\Lambda \Omega_\Lambda).
\end{eqnarray*}}
{
	\begin{eqnarray*}
		\mathbb{A}_k(\Omega) &=& \E y y^\star\chi_{\{y\colon k\in J^\Omega(y)\}}\\&=& \E (\id - \Omega^\dagger_\Lambda \Omega_\Lambda) z z^\star(\id - \Omega^\dagger_\Lambda \Omega_\Lambda)\chi_{\{(\Lambda,z)\colon k\in J_{(\Lambda,z)}\}}\\
		&=&\E_\Lambda (\id - \Omega^\dagger_\Lambda \Omega_\Lambda)\E_z z z^\star(\id - \Omega^\dagger_\Lambda \Omega_\Lambda)\chi_{\{\Lambda\colon k\in\Lambda\}}\\
		&=&\E_\Lambda  (\id - \Omega^\dagger_\Lambda \Omega_\Lambda)\chi_{\{\Lambda\colon k\in\Lambda\}}\\
		&=&{K\choose\ell}^{-1}\sum_{\Lambda\colon k\in\Lambda}  (\id - \Omega^\dagger_\Lambda \Omega_\Lambda),
\end{eqnarray*}}

where $J_{(\Lambda,z)}= \argmin_{|J| = \ell} \Vert \Omega_J (\id - \Omega^\dagger_\Lambda \Omega_\Lambda) z\Vert_2^2$.

Multiplying this matrix with $\omega_k$ yields zero, as $k$ always lies in $\Lambda$ and so every term in the sum maps $\omega_k$ to zero.

The Analysis K-SVD algorithm~\cite{rupeel13} takes a similar approach as the SVAOL algorithm. At first, cosupport sets are estimated and then the singular vector to the smallest singular value of the resulting data matrix is computed. The notable difference, however, is how the cosupport set is estimated. We use a hard thresholding approach, whereas for Analysis K-SVD an analysis sparse-coding step is employed, which uses significantly more computational resources.

We see that the computation and storage complexity for the first steps (i.e. setting up and storing the matrices $A_k$) of the SVAOL algorithm are the same as for IAOL. This means that the spatial complexity of SVAOL is $\Ocal(Kd^2)$. For the runtime complexity, which is $\Ocal(NKd)$ for the setup of the matrices, we need to include the cost of computing the $K$ smallest singular values of the matrices $A_k$. This can be done for example using the SSVII algorithm presented in~\cite{schwetlick2003iterative}, in which for each iteration a system of linear equations of size $d+1$ has to be solved. We observed that typically 10 iterations of the algorithm are sufficient, so the computational complexity for this step is 10 times higher than for IAOL.

%%%%%%%%%%%%%%%%%%%%%%
\subsection{Experiments on synthetic data}\label{sec:comparison}
%%%%%%%%%%%%%%%%%%%%%
As for the explicit algorithms presented above, we first try our new algorithm on synthetic data. For this, we again learn an operator from data generated from a random operator with normalised rows in $\R^{128\times 64}$. The setup is the same as in Section~\ref{sec:StupidNum} and the results are shown in Figure~\ref{fig:BackwardNoisefree}. We use a large stepsize $\alpha = 100$ in order to achieve fast convergence.

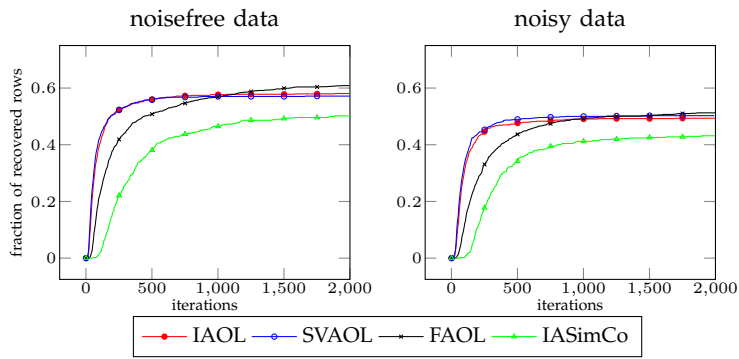
\begin{figure}
	\centering
	\begin{tikzpicture}
	\begin{groupplot}[
	group style={group size=2 by 1, horizontal sep = 1cm}]
	\nextgroupplot[title = {\footnotesize noisefree data}, title style={yshift=-.5ex},legend to name={CommonLegend},legend style={legend columns=4, font = \scriptsize},
	xmax = 2000,
	xlabel=iterations,
	ylabel= fraction of recovered rows,
	yticklabel style = {font=\tiny,xshift=0.5ex},
	xticklabel style = {font=\tiny,yshift=0.5ex},
	ylabel style = {font=\tiny,yshift =-.65cm},
	xlabel style = {font=\tiny,yshift = .4cm},
	ymax=0.75, 
	width = 0.3\columnwidth]
	
	\addplot[color=red,mark = *,mark size=1pt, mark repeat = 50] table[x=t,y=l55n0]{NewPics/IAOL_Synthetic_Noisefree.txt};  
	\addplot[color=blue,mark = o,mark size=1pt, mark repeat = 50] table[x=t,y=l55n0]{NewPics/MEAOL_Synthetic_Noisefree.txt}; 
	\addplot[color=black,mark = x,mark size=1pt, mark repeat = 50] table[x=t,y=l55n0]{NewPics/FAOL_Synthetic_Noisefree.txt}; 
	\addplot[color=green,mark = triangle,mark size=1pt, mark repeat = 50] table[x=t,y=l55n0]{NewPics/ASimCO_Synthetic_Noisefree.txt}; 
	
	\addlegendimage{red, mark=*}
	\addlegendentry{IAOL}
	\addlegendentry{SVAOL}
	\addlegendentry{FAOL}
	\addlegendentry{IASimCo}
	
	\nextgroupplot[title = {\footnotesize noisy data},
	title style={yshift=-1ex},
	xmax = 2000,
	xlabel=iterations,
	yticklabel style = {font=\tiny,xshift=0.5ex},
	xticklabel style = {font=\tiny,yshift=0.5ex},
	ylabel style = {font=\tiny,yshift =-.65cm},
	xlabel style = {font=\tiny,yshift = .4cm},
	ymax=0.75, 
	width = 0.3\columnwidth
	]
	\addplot[color=red,,mark = *,mark size=1pt, mark repeat = 50] table[x=t,y=l55n02]{NewPics/IAOL_Synthetic_Noisy.txt};  
	\addplot[color=blue, ,mark = o,mark size=1pt, mark repeat = 50] table[x=t,y=l55n02]{NewPics/MEAOL_Synthetic_Noisy.txt}; 
	\addplot[color=black,mark = x,mark size=1pt, mark repeat = 50] table[x=t,y=l55n02]{NewPics/FAOL_Synthetic_Noisy.txt}; 
	\addplot[color=green,mark = triangle,mark size=1pt, mark repeat = 50] table[x=t,y=l55n02]{NewPics/ASimCO_Synthetic_Noisy.txt}; 
	
	\end{groupplot}
	\path (group c1r1.south east) -- node[below=10pt]{\ref{CommonLegend}} (group c2r1.south west);
	\end{tikzpicture}
	\caption{Recovery rates of IAOL and SVAOL in a noisefree (left) and a noisy (right) setting compared to FAOL and ASimCO using cosparsity level $\ell = 55$ using a random initialisation.}
	\label{fig:BackwardNoisefree}
\end{figure}

Note that IAOL and SVAOL saturate faster than FAOL, cf. Figure~\ref{fig:StupidNoisefree}. However, IAOL and SVAOL without replacement recover slightly fewer  rows as FAOL, which is probably a result of the faster convergence speed.\\
Finally, since the implicit algorithms per se, like FAOL, do not penalise the recovery of two identical rows, cf. Figure~\ref{fig:MultipleRecovery}, we again need to use the replacement strategy introduced in Section~\ref{sec:ReplacementSimple}. 

\begin{figure}
	\centering
	\begin{tikzpicture}
	\begin{groupplot}[
	group style={group size=2 by 1, horizontal sep = 1cm}]
	\nextgroupplot[title = {\footnotesize noisefree data}, title style={yshift=-.5ex},legend to name={CommonLegend},legend style={legend columns=4, font = \scriptsize},
	xlabel=iterations,
	ylabel= fraction of recovered rows,
	yticklabel style = {font=\tiny,xshift=0.5ex},
	xticklabel style = {font=\tiny,yshift=0.5ex},
	ylabel style = {font=\tiny,yshift =-.65cm},
	xlabel style = {font=\tiny,yshift = .4cm},
	width = 0.3\columnwidth, 
	xtick = {0,5000,10000},
	xticklabels = {$0$,$5000$,$10000$},
	ymax =1.05, ymin = 0]
	\addplot[color=red, mark = *, mark size = 1pt, mark repeat = 200] table[x=t,y=targIAOL]{NewPics/CompareAllSyntheticNoisefreeReplacement.txt};  
	\addplot[color=blue, mark = o, mark size = 1pt, mark repeat = 200] table[x=t,y=targMEAOL]{NewPics/CompareAllSyntheticNoisefreeReplacement.txt};
	\addplot[color=black, mark = x, mark size = 1pt, mark repeat = 200] table[x=t,y=targFAOL]{NewPics/CompareAllSyntheticNoisefreeReplacement.txt};
	\addplot[color=green,mark = triangle,mark size=1pt, mark repeat = 200] table[x=t,y=targSimCo]{NewPics/CompareAllSyntheticNoisefreeReplacement.txt}; 
	\addlegendimage{red, mark=*}
	\addlegendentry{IAOL}
	\addlegendentry{SVAOL}
	\addlegendentry{FAOL}
	\addlegendentry{IASimCo}
	
	\nextgroupplot[title = {\footnotesize noisy data},
	title style={yshift=-1ex},
	width=0.3\columnwidth,
	xlabel=iterations,
	yticklabel style = {font=\tiny,xshift=0.5ex},
	xticklabel style = {font=\tiny,yshift=0.5ex},
	ylabel style = {font=\tiny,yshift =-.65cm},
	xlabel style = {font=\tiny,yshift = .4cm},
	ymax =1.05, 
	xtick = {0,5000,10000},
	xticklabels = {$0$,$5000$,$10000$},
	ymin = 0
	]
	\addplot[color=red, mark = *, mark size = 1pt, mark repeat = 200] table[x=t,y=targIAOL]{NewPics/CompareAllSyntheticNoisyReplacement.txt};  
	\addplot[color=blue, mark = o, mark size = 1pt, mark repeat = 200] table[x=t,y=targMEAOL]{NewPics/CompareAllSyntheticNoisyReplacement.txt}; 
	\addplot[color=black, mark = x, mark size = 1pt, mark repeat = 200] table[x=t,y=targFAOL]{NewPics/CompareAllSyntheticNoisyReplacement.txt};  
	\addplot[color=green,mark = triangle,mark size=1pt, mark repeat = 200] table[x=t,y=targSimCo]{NewPics/CompareAllSyntheticNoisyReplacement.txt}; 
	
	\end{groupplot}
	\path (group c1r1.south east) -- node[below=10pt]{\ref{CommonLegend}} (group c2r1.south west);
	\end{tikzpicture}
	\caption{Recovery rates of IAOL, SVAOL, FAOL and IASimCo with replacement from signals with cosparsity level $\ell=55$ in a noiseless (left) and a noisy (right) setting, using a random initialisation.}
	\label{fig:ReplacementBackward}
\end{figure}
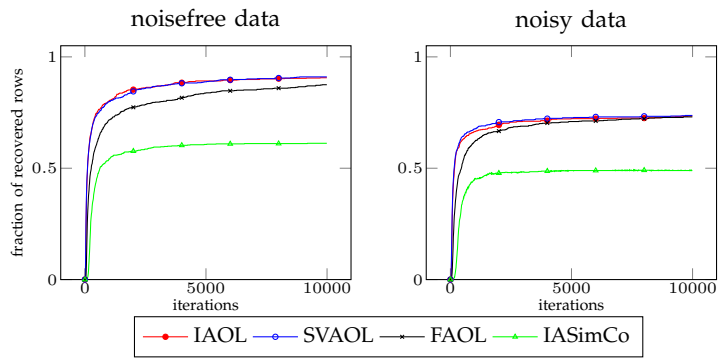
 The simulation results, using replacement with $\mu_0=0.8$ and the usual setup are shown in Figure~\ref{fig:ReplacementBackward}. We see that IAOL and SVAOL come closer to full recovery than their explicit counterpart FAOL within the considered 10000 iterations. Again, for noisy data the algorithms saturate well below full recovery.

%%%%%%%%%%%%%%%%%%%%%%%
\subsection{Experiments on image data}
%%%%%%%%%%%%%%%%%%%%%%
Finally, we want to see how the stabilised algorithms perform on real data. We use the same image (Shepp Logan) and setup as in Section~\ref{sec:StupidRealData} to learn a square analysis operator for $8\times 8$ patches, cf. Figure~\ref{fig:SAOLFAOLRealData}. We will not use the SAOL algorithm in the simulations from now on, as the execution time in Matlab is considerably higher due to the required for-loops. However, as we have seen, it performs mostly like the FAOL algorithm.\\
\ifthenelse{\boolean{onecol}}{
\begin{figure}[h!]
	\centering
	\begin{tikzpicture}
	\begin{axis}[xlabel=iterations,
	yticklabel style = {font=\tiny,xshift=0.5ex},
	xticklabel style = {font=\tiny,yshift=0.5ex},
	ylabel style = {font=\tiny,yshift =-.65cm},
	xlabel style = {font=\tiny,yshift = .4cm},
	legend style = {font=\tiny, draw = none},
	legend pos=north east, 
	width = 0.3\textwidth, 
	ymode = log, 
	ytick = {0.001,0.0004},
	yticklabels = {$10^{-3}$,$4\cdot 10^{-4}$},
	ymax = 0.001,
	ymin = 0.0004,
	ylabel= value of target function]
	\addplot[color=red, mark = *, mark repeat = 10, mark size = 1pt] table[x=t,y=targIAOL]{NewPics/DecayTargetCompareExplicitImplicit.txt};  
	\addplot[color=blue,mark = o, mark repeat = 10, mark size = 1pt] table[x=t,y=targMEAOL]{NewPics/DecayTargetCompareExplicitImplicit.txt}; 	\addplot[color=black, mark = x, mark repeat = 10, mark size = 1pt] table[x=t,y=targFAOL]{NewPics/DecayTargetCompareExplicitImplicit.txt}; 	
	\legend{IAOL,SVAOL, FAOL}
	\end{axis}
	\end{tikzpicture}	
	\quad
	\includegraphics[width = 0.25\textwidth,height=0.25\textwidth]{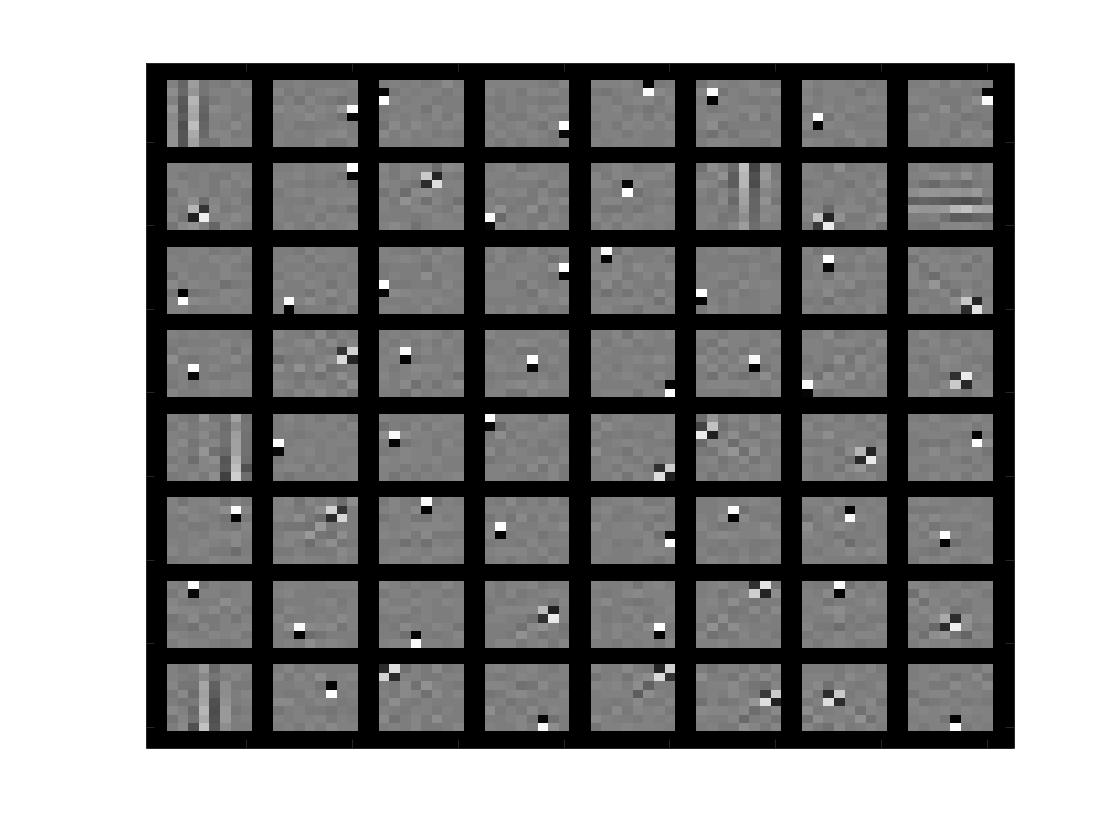}
	
	\caption{Decay of the target function using FAOL, IAOL and SVAOL for the Shepp Logan phantom (left) and the operator recovered by IAOL after 100(!) iterations (right).}
	\label{fig:LearnedOpISAOL}
\end{figure}}
{
	\begin{figure}[h!]
		\centering
		\begin{tikzpicture}
		\begin{axis}[xlabel=iterations,
		yticklabel style = {font=\tiny,xshift=0.5ex},
		xticklabel style = {font=\tiny,yshift=0.5ex},
		ylabel style = {font=\tiny,yshift =-.65cm},
		xlabel style = {font=\tiny,yshift = .4cm},
		legend style = {font=\tiny},
		legend pos=north east, 
		width = 0.55\columnwidth, 
		ymode = log, 
		ytick = {0.001,0.0004},
		yticklabels = {$10^{-3}$,$4\cdot 10^{-4}$},
		ymax = 0.001,
		ymin = 0.0004,
		ylabel= value of target function]
		\addplot[color=red, mark = *, mark repeat = 10, mark size = 1pt] table[x=t,y=targIAOL]{NewPics/DecayTargetCompareExplicitImplicit.txt};  
		\addplot[color=blue,mark = o, mark repeat = 10, mark size = 1pt] table[x=t,y=targMEAOL]{NewPics/DecayTargetCompareExplicitImplicit.txt}; 	\addplot[color=black,mark = x, mark repeat = 10, mark size = 1pt] table[x=t,y=targFAOL]{NewPics/DecayTargetCompareExplicitImplicit.txt}; 	
		\legend{IAOL,SVAOL, FAOL}
		\end{axis}
		\end{tikzpicture}	
		\quad
		\includegraphics[width = 0.22\textwidth,height=0.2\textwidth]{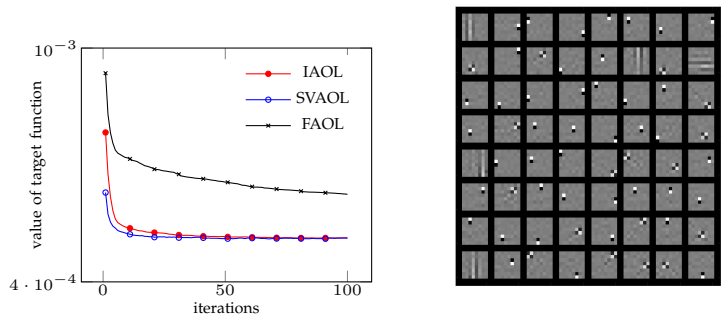}
		
		\caption{Decay of the target function using FAOL, IAOL and SVAOL for the Shepp Logan phantom (left) and the operator recovered by IAOL after 100(!) iterations (right).}
		\label{fig:LearnedOpISAOL}
\end{figure}}

As can be seen in Figure~\ref{fig:LearnedOpISAOL}, the training is much faster now, because the stepsize does not have to be chosen as small as in the previous section. The decrease in the objective function is very fast compared to FAOL, and we see that already after a few iterations the algorithm stabilises and we, as expected, obtain combinations of discrete gradients as anlysers. As for FAOL we observe that the replacement strategy for $\mu_0=0.99$ is hardly ever activated and that lowering the threshold results in finding and replacing the same translated edge detectors.\\
In the remainder of this section, we will investigate the time complexity of the presented algorithms numerically. Naturally, the explicit algorithms use significantly fewer computational resources per iteration. We compare the average calculation times per iteration on a  3.1 GHz Intel Core i7 Processor. For IASimCO the out-of-the-box version on the homepage of the authors of~\cite{dowadaplha16} is used.

Figure~\ref{fig:ComparisonAlgorithmsRealData} (left) shows the recovery rates of the four algorithms plotted against the cumulative execution time. We can see that on synthetic data, the IAOL and SVAOL algorithms show a similar performance to FAOL, followed by IASimCo.  Using FAOL as a baseline, we see that one iteration of ASimCo takes about twice as long, one iteration of IAOL takes about four times as long and one iteration of SVAOL takes 5-6 times longer, but this significantly depends on the number of iterations of the inverse iteration to find the smallest singular value. 

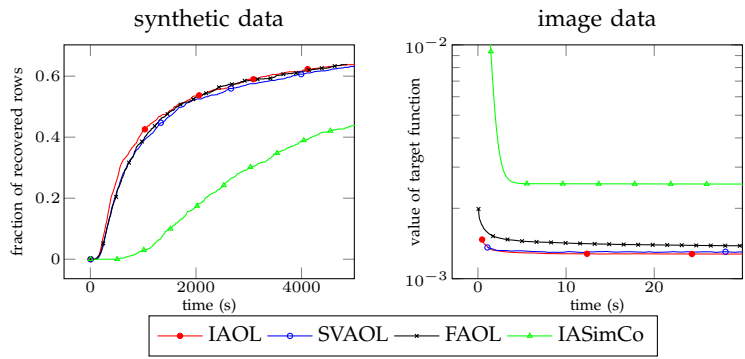
\begin{figure}
\centering
\begin{tikzpicture}
\begin{groupplot}[
group style={group size=2 by 1, horizontal sep = 1.3cm}]
\nextgroupplot[title = {\footnotesize synthetic data}, title style={yshift=-1ex},legend to name={CommonLegend},legend style={legend columns=4, font = \scriptsize},
xlabel=time (s),
yticklabel style = {font=\tiny,xshift=0.5ex},
xticklabel style = {font=\tiny,yshift=0.5ex},
ylabel style = {font=\tiny,yshift =-.6cm},
xlabel style = {font=\tiny,yshift = .4cm},
ylabel= fraction of recovered rows,
width=0.3\columnwidth,
xmax = 5000,
xtick={0,2000,4000},
xticklabels={$0$,$2000$,$4000$}]
\addplot[color=red,mark = *, mark size = 1pt, mark repeat = 50] table[x=tIAOL,y=hitsIAOL]{NewPics/CompareAllSyntheticNoisefreeImprovedTimes.txt};  	
\addplot[color=blue,mark = o, mark size = 1pt, mark repeat = 50] table[x=tSVAOL,y=hitsSVAOL]{NewPics/CompareAllSyntheticNoisefreeImprovedTimes.txt};
\addplot[color=black,mark = x, mark size = 1pt, mark repeat = 50] table[x=tFAOL,y=hitsFAOL]{NewPics/CompareAllSyntheticNoisefreeImprovedTimes.txt};    	 	 
\addplot[color=green,mark = triangle, mark size = 1pt, mark repeat = 50] table[x=tSimCo,y=hitsSimCo]{NewPics/CompareAllSyntheticNoisefreeImprovedTimes.txt};
\addlegendimage{red, mark=*}
\addlegendentry{IAOL}
\addlegendentry{SVAOL}
\addlegendentry{FAOL}
\addlegendentry{IASimCo}

\nextgroupplot[title = {\footnotesize image data},
title style={yshift=-1ex},
xlabel=time (s),
yticklabel style = {font=\tiny,xshift=0.5ex},
xticklabel style = {font=\tiny,yshift=0.5ex},
ylabel style = {font=\tiny,yshift =-.75cm},
xlabel style = {font=\tiny,yshift = .4cm},
ymode = log, ylabel= value of target function,
width=0.3\columnwidth,
xmax = 30,
yminorticks=true,
xtick={0,10,20},
xticklabels={$0$,$10$,$20$},
ymin = 1e-3,ymax = 1e-2
]
\addplot[color=red,mark = *, mark size = 1pt, mark repeat = 50] table[x=tIAOL,y=hitsIAOL]{NewPics/CompareAllSheppLoganImprovedTimes.txt};
\addplot[color=blue,mark = o, mark size = 1pt, mark repeat = 50] table[x=tSVAOL,y=hitsSVAOL]{NewPics/CompareAllSheppLoganImprovedTimes.txt};
\addplot[color=black,mark = x, mark size = 1pt, mark repeat = 50] table[x=tFAOL,y=hitsFAOL]{NewPics/CompareAllSheppLoganImprovedTimes.txt};  
\addplot[color=green,mark = triangle, mark size = 1pt, mark repeat = 50] table[x=tSimCo,y=hitsSimCo]{NewPics/CompareAllSheppLoganImprovedTimes.txt};

\end{groupplot}
\path (group c1r1.south east) -- node[below=10pt]{\ref{CommonLegend}} (group c2r1.south west);
\end{tikzpicture}
\caption{Recovery rates of FAOL, IAOL, SVAOL and IASimCo from 55-cosparse signals in a noisy setting (left). Decay of the target function using IASimCO, FAOL, IAOL and SVAOL to learn a $128\times 64$ operator for the Shepp Logan image (right). The figures depict the execution time of the algorithm on the x-axis. SAOL is omitted, as the execution time is not comparable due to the Matlab implementation.}
\label{fig:ComparisonAlgorithmsRealData}
\end{figure}

In the experiment on image data, we learn an overcomplete operator with $128$ rows from the $8\times8$ patches of the $256\times 256$ (unnormalised) Shepp Logan phantom contamined with Gaussian noise with $\operatorname{PSNR} \approx 20$. We choose as cosparsity level $\ell=120$, initialise randomly and in each iteration use $20000$ randomly selected patches out of the available 62001. Since for image data our replacement strategy is hardly ever activated, we directly omit it to save computation time. IASimCo is again used in its out-of-the-box version.
In Figure~\ref{fig:ComparisonAlgorithmsRealData}, one can clearly observe that the IAOL and SVAOL algorithms indeed minimise the target function in a fraction of the iterations necessary for the FAOL algorithm, which in turn is much faster than IASimCO. Already after $10$ seconds, IAOL and SVAOL have essentially finished minimising the objective function, whereas FAOL needs, as seen in the previous section, about 100000 iterations to get to approximately the same value of the objective function. IASimCo lags behind severely and, as indicated by the shape of the curve, saturates at a highly suboptimal value of the target function. This fact can also be observed by looking at the learned analysis operators in Figure~\ref{fig:Recovered100}.\\
\begin{figure*}
	\centering 
	\includegraphics[width = 0.8\textwidth]{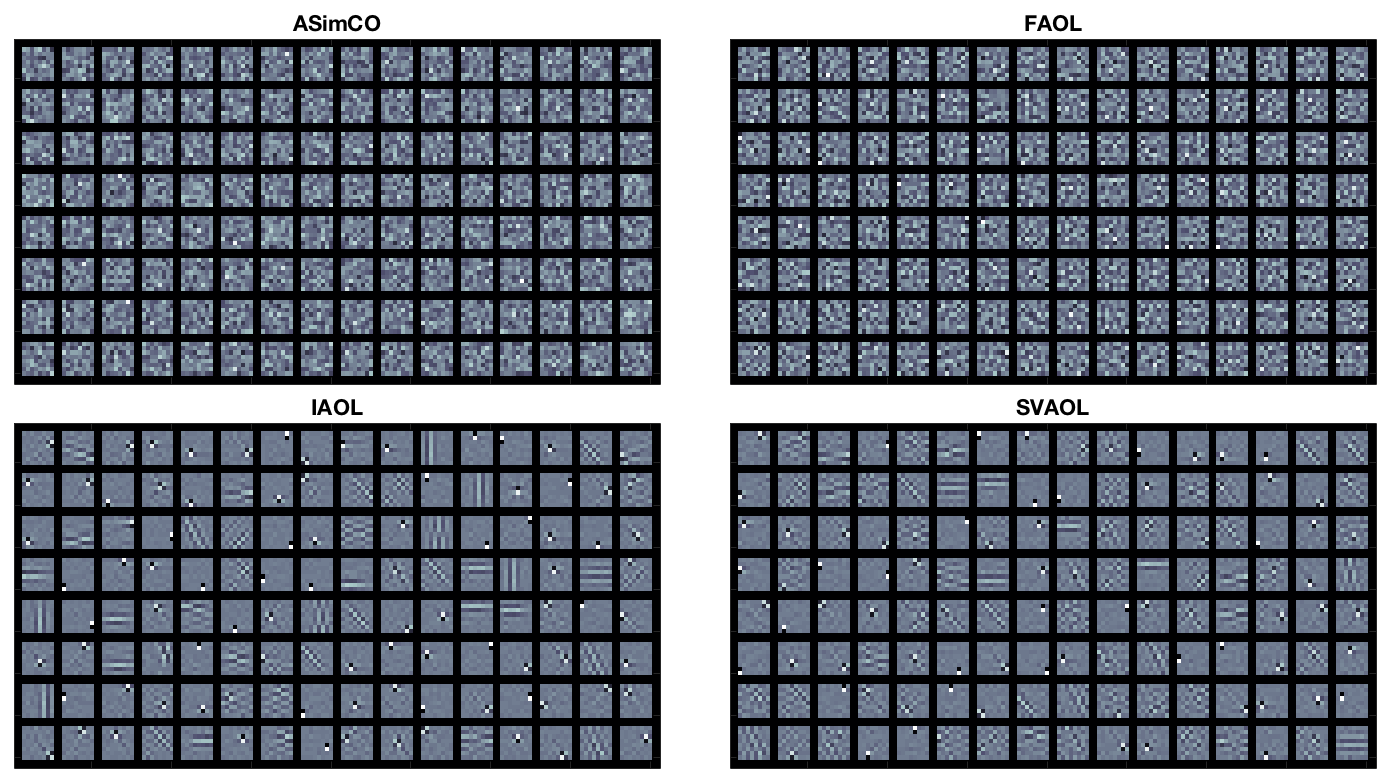}
	\caption{Analysis operators learned in 120 seconds with data drawn from the Shepp-Logan phantom contaminated with Gaussian noise, by ASimco (top left), FAOL (top right), IAOL (bottom left) and SVAOL (bottom right).}
	\label{fig:Recovered100}
\end{figure*}
Encouraged by this good performance we will in the next section apply our algorithms to image denoising.

%%%%%%%%%%%%%%
\section{Image denoising}\label{sec:denoise}
%%%%%%%%%%%%%%

In this section we will compare the performance of analysis operators learned by the FAOL, IAOL and SVAOL algorithms presented in this paper in combination with Tikhonov regularisation for image denoising to the performance of operators learned by (I)ASimCO. For easy comparison we use the same setup as in~\cite{dowadaplha16}, where (I)ASimCo is compared to several other major algorithms for analysis operator learning, \cite{rabr13, yanagrda13,hakldi13,rupeel13,ekba14}, and found to give the best performance. \\
%%%%%%%%
\begin{figure}
	\centering 
	\includegraphics[width = 0.2\textwidth]{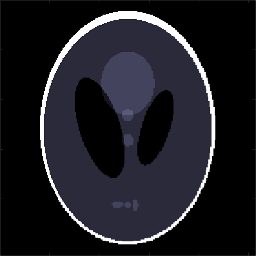}\,\,
	\includegraphics[width = 0.2\textwidth]{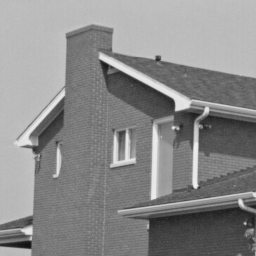}\,\,
	\includegraphics[width = 0.2\textwidth]{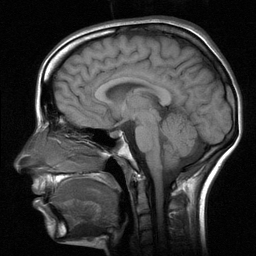}\\\vspace{6pt}
	\includegraphics[width = 0.2\textwidth]{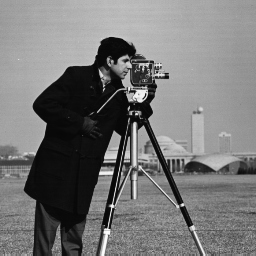}\,\,
	\includegraphics[width = 0.2\textwidth]{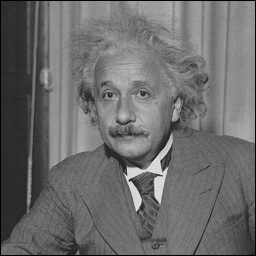}\,\,
	\includegraphics[width = 0.2\textwidth]{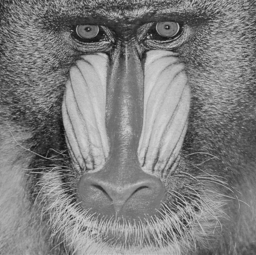}
	\caption{Images used for learning and denoising. Top: Shepp-Logan, House, MRI; Bottom: Cameraman, Einstein, Mandrill.}
	\label{fig:TrainingImages}
\end{figure}
%%%%%%%%
{\bf Learning setup:} We follow the setup for the Shepp-Logan image in the last section. Our training data consists of all $8\times 8$ patches of one of the $256\times 256$ images from Figure~\ref{fig:TrainingImages} corrupted with Gaussian white noise of standard deviation $\sigma = 12.8$ and $\sigma = 45$ leading to a PSNR of approximately 25dB and 15dB, respectively. The analysis operators of size $128\times 64$ are initialised by drawing each row uniformly at random from the unit sphere, and then updated using in each step 20000 randomly selected patches of the available 62001 and a cosparsity level $\ell \in \{70,80,90,100,110,120\}$. The same initialisation is used for all algorithms. For (I)ASimCo and FAOL we use 2000 and for IAOL and SVAOL 500 iterations. We perform the optimisation without replacement for FAOL, IAOL and SVAOL. \\
{\bf Denoising setup:} For the denoising step we use a standard approach via Tikhonov regularisation based on the learned analysis operator $\Gamma$, \cite{elah06,elmiru07}. For each noisy patch $y$ we solve,
\begin{align}
\hat y = \argmin_z\, \lambda \Vert \Gamma z\Vert_1+\Vert z-y\Vert_2
\end{align}
for a regularisation parameter $\lambda\in\{0.002,0.01, 0.05,0.1,0.3,0.5\}$.  
We then reassemble the denoised patches $\hat y$ to the denoised image, by averaging each pixel in the full image over the denoised patches in which it is contained. To measure the quality of the reconstruction for each cosparsity level $\ell$ and regularisation parameter $\lambda$ we  average the PSNR of the denoised image over 5 different noise realisations and initialisations. Table~\ref{tab:ComparisonDenoising} shows the PSNR for optimal choice of $\ell$ and $\lambda$ for each of the algorithms.
\begin{table}
	\centering
	\begin{tabular}{|c|c|c|c|c|c|c|c|}
		\hline 
		&Algorithm&SL&Cam&Ein&MRI&Hou&Man\\ 
		\hline \hline
		\multirow{ 5}{*}{\rot{$\sigma =12.8$}}&ASimCO &32.57&{\bf30.36}&31.30&31.69&31.78&28.52\\ 
		\cline{2-8} 
		&IASimCO&32.49&30.32&31.19&{\bf31.77}&31.60&28.28\\ 
		\hhline{|~|=|=|=|=|=|=|=|}
		&FAOL&{\bf 33.91}&30.33&{\bf 31.62}&31.70&{\bf 32.86}&{\bf 28.71}\\ 
		\cline{2-8}  
		&IAOL&33.39&30.24&31.54&31.71&32.66&28.64\\ 
		\cline{2-8} 
		&SVAOL&33.49&30.22&31.57&31.73&32.77&28.70\\ 
		\hline \hline
		\multirow{5}{*}{\rot{$\sigma = 45$}}&ASimCO&27.65&24.16&{\bf25.87}&{\bf25.66}&27.19&23.46\\ 
		\cline{2-8}  
		&IASimCO&27.50&23.87&25.78&25.52&27.12&23.33\\ 
		\hhline{|~|=|=|=|=|=|=|=|}
		&FAOL&28.33&24.17&25.82&{\bf25.66}&{\bf 27.21}&23.50\\ 
		\cline{2-8}  
		&IAOL&{\bf28.38}&24.18&25.83&25.65&27.19&{\bf23.52}\\ 
		\cline{2-8} 
		&SVAOL&28.36&{\bf24.20}&25.84&25.62&27.18&23.51\\ 
		\hline 
	\end{tabular} 
	\caption{Performance of FAOL, IAOL, SVAOL and (I)ASimCO for denoising for different pictures and noise levels $\sigma$. }
	\label{tab:ComparisonDenoising}
\end{table}
We can see that all five algorithms provide a comparable denoising performance, mostly staying with 0.1dB of each other. However, while FAOL, IAOL, SVAOL never lag more than 0.14dB behind, they do improve upon (I)ASimCo for more than 1dB twice. The denoising results for one of these cases, that is the Shepp-Logan phantom in the low-noise regime, are shown in Figure~\ref{fig:DenoiseSL12} .

\begin{figure}
	\centering 
	\includegraphics[width = 0.3\columnwidth]{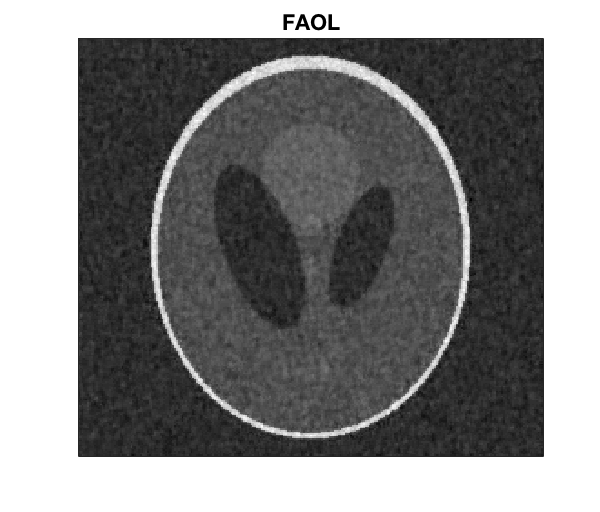}\,
	\includegraphics[width = 0.3\columnwidth]{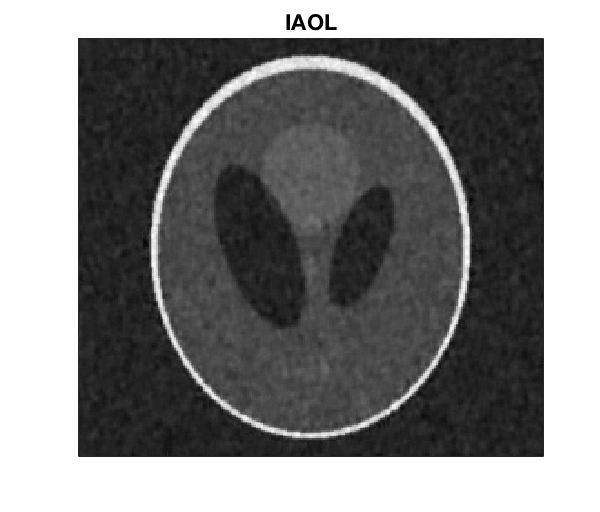}\\
	\includegraphics[width = 0.3\columnwidth]{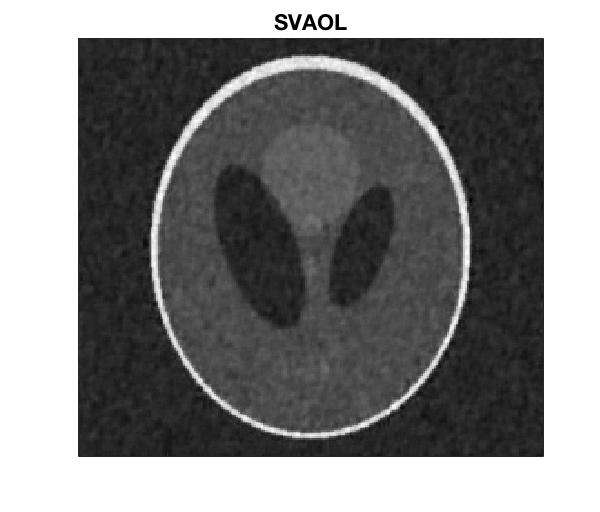}\,
	\includegraphics[width = 0.3\columnwidth]{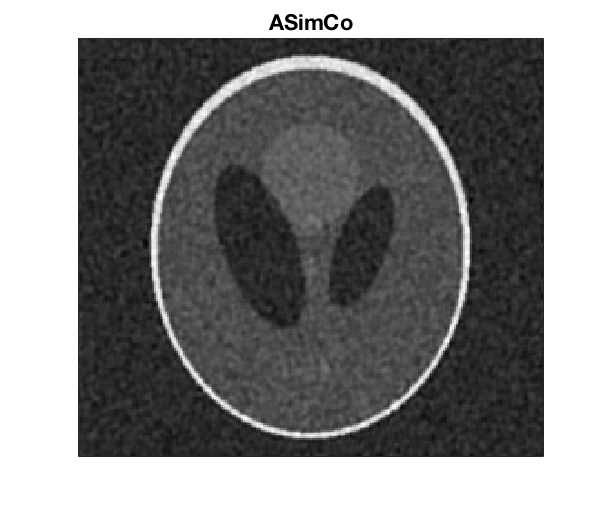}
	\caption{Denoised Shepp-Logan phantom using the optimal parameters for the various presented algorithms for noise level $\sigma=12.8$.}
	\label{fig:DenoiseSL12}
\end{figure}

After confirming that our algorithms indeed learn useful operators also on real data, we now turn to a discussion of our results.

%%%%%%%%%%%%%%
\section{Discussion}\label{sec:disc}
%%%%%%%%%%%%%%
We have developed four algorithms for analysis operator learning based on projected stochastic gradient-like descent, SAOL, FAOL, IAOL and SVAOL.
The algorithms perform better than the state-of-the-art algorithms (I)ASimCO, \cite{dowadaplha16}, which are similarly gradient descent based and have slightly higher but comparable computational complexity per iteration, in terms of recovery rates resp. reduction of the objective function. Another advantage of SAOL is that it is sequential with a memory requirement corresponding to the size of the operator, $\Ocal(dK)$. In contrast ASimCO either is non-sequential with a memory requirement of the order of the data matrix, $\Ocal(dN)$, or in a sequential setting needs $O(LN)$ training samples corresponding to the $L$ evaluations of the objective function necessary for the line search. IAOL and SVAOL, which are more stable than SAOL and FAOL, are sequential when accepting a memory requirement $\Ocal(d^2K)$ and in a non-sequential setting have again memory requirement $\Ocal(dN)$.\\
On synthetic data, the recovery of the target operator using the algorithms presented here is significantly faster than with (I)ASimCo. On real data the implicit algorithms IAOL and SVAOL minimise the objective function in a fraction of the time that is needed by the explicit algorithms FAOL and ASimCo.
Considering image denoising via Tikhonov regularisation as application of analysis operator learning, we see that the operators presented in this paper give similar or better results as the (I)ASimCo operators in the considered denoising setups. \\
A Matlab toolbox to reproduce all the experiments reported in this paper can be found at \url{http://homepage.uibk.ac.at/~c7021041/code/AOL.zip}.\\
While the good performance of the developed algorithms certainly justified the effort, one of our main motivations for considering a projected gradient descent approach to analysis operator learning was to derive convergence results similar to those for dictionary learning, \cite{sc15}.
However, even a local convergence analysis, turns out to be quite different and much more complicated than for dictionary learning. The main reason for this is that sparsity is more robust to perturbations than cosparsity. So for an $S$-sparse signal $y=\dico_I x_I$ and a perturbed dictionary $\pdico$ with $\|\patom_k-\atom_k\|_2<\eps$ for balanced $x_I$ the best $S$-term approximation in $\pdico$ will still use the same support $I$. In contrast, if $y$ is $\ell$-cosparse with respect to an analysis operator $\Omega$, $\Omega_\Lambda y =0$, then for a perturbed operator $\Gamma$ with $\|\gamma_k-\omega_k\|_2<\eps$ the smallest $\ell$ entries of $\Gamma y$ will not all be located in $\Lambda$. In order to get a local convergence result one has to deal with the fact that only part of the cosupport is preserved. We expect that for most signals containing $k$ in the cosupport with respect to $\Omega$, $k$ will also be in the cosupport with respect to $\Gamma$. Unfortunately the mathematical tools necessary to quantify these statements are much more involved than the comparatively simple results necessary for the convergence of dictionary learning and so the local convergence analysis remains on our agenda for future research. It is also possible that the analysis can be carried out from a dynamical systems perspective using the differential equations described in the beginning of Section~\ref{sec:backward}.\\
Another research direction we are currently pursuing is inspired by the shape of the analysis operators learned on noiseless images. The translation invariance of the edge detector like analysers suggests to directly assume translation invariance of the analysis operator. Such an operator has two advantages, first, learning it will require less training samples and second, since it can be reduced to several translated mother functions, it will be cost efficient to store and apply.

\section*{Acknowledgements}
This work was supported by the Austrian Science Fund (FWF) under Grant no.~Y760.
In addition the computational results presented have been achieved (in part) using the HPC infrastructure LEO of the University of Innsbruck.
Part of this work has been carried out while M.S. was supported by the trimester program 'Mathematics of Signal Processing' at the Hausdorff Research Institute for Mathematics.
The authors would also like to thank Marie Pali for proof-reading the manuscript and her valuable comments.

\appendix

\section{Computation of the optimal stepsize} \label{sec:CompAlpha}

To find the optimal stepsize, we first compute the derivative of \[F(\alpha) = \frac{(\gamma_k -\alpha g_k)A_k(\gamma_k - \alpha g_k)^\star}{\|(\gamma_k -\alpha g_k)\|^2},\]which is given by \[F'(\alpha) = -2\frac{(b-a^2)+\alpha(ab - c) + \alpha^2(ac-b^2)}{(1-2\alpha a +\alpha^2 b^2)^2},\] where we used the notation $a=\gamma_k A_k\gamma_k^\star$, $b=\gamma_k A_k^2\gamma_k^\star$ and $c=\gamma_k A_k^3\gamma_k^\star$.

First, suppose that $b^2 \neq ac$. Setting the first derivative equal to zero and solving a quadratic equation gives the results \begin{equation}\label{eq:alphaPM}\alpha_{\pm}  = \frac{ab-c\pm\sqrt{(c-ab)^2 - 4(b^2-ac)(a^2-b)}}{2(b^2-ac)}.\end{equation}

The discriminant is always larger or equal than zero, as \[(c-ab)^2 - 4(b^2-ac)(a^2-b) = (c+2a^3-3ab)^2 - 4(a^2-b)^3\] and $a^2 -b = (\gamma_k A_k \gamma_k^\star)^2 - \gamma_k A_k^2 \gamma_k^\star = \gamma_k A_k(\gamma_k^\star \gamma_k -\id)A_k \gamma_k^\star \leq0$, because the matrix $\gamma_k^\star\gamma_k -\id$ is negative semidefinite. 

We can verify that $\alpha_+$ indeed minimizes F, by substituting it into the second derivative
\ifthenelse{\boolean{onecol}}{
\begin{align*}F''(\alpha) = -2\tfrac{(ab-c + 2\alpha(ac-b^2))(1-2\alpha a +\alpha^2 b^2)^2-4(1-2\alpha a + \alpha^2 b^2)(\alpha b - a)(b-a^2+\alpha(ab-c)+\alpha^2(ac-b^2))}{(1-2\alpha a+\alpha^2 b^2)^4}.\end{align*}}
{\begin{align*}F''(\alpha) = -2\tfrac{(ab-c + 2\alpha(ac-b^2))(1-2\alpha a +\alpha^2 b^2)^2}{(1-2\alpha a+\alpha^2 b^2)^4}+\\
	+\tfrac{8(1-2\alpha a + \alpha^2 b^2)(\alpha b - a)(b-a^2+\alpha(ab-c)+\alpha^2(ac-b^2))}{(1-2\alpha a+\alpha^2 b^2)^4}.
	\end{align*}}
We see that
\ifthenelse{\boolean{onecol}}{
\[F''(\alpha_+) = -2\tfrac{(ab-c + 2\alpha_+(ac-b^2))(1-2\alpha_+ a +\alpha_+^2 b^2)^2}{(1-2\alpha_+ a+\alpha_+^2 b^2)^4}  = -2\tfrac{ab-c + 2\alpha_+(ac-b^2)}{(1-2\alpha_+ a+\alpha_+^2 b^2)^2}.\]}
{\begin{align*}
	F''(\alpha_+) &=& -2\tfrac{(ab-c + 2\alpha_+(ac-b^2))(1-2\alpha_+ a +\alpha_+^2 b^2)^2}{(1-2\alpha_+ a+\alpha_+^2 b^2)^4} =\\ &=& -2\tfrac{ab-c + 2\alpha_+(ac-b^2)}{(1-2\alpha_+ a+\alpha_+^2 b^2)^2}.
	\end{align*}}
The denominator of the fraction above is positive, so we need to show that the numerator is negative. Inserting the expression for $\alpha_+$ from Equation~\eqref{eq:alphaPM} into the numerator yields
\begin{align*}
ab-c - 2(b^2-ac)\tfrac{ab-c+\sqrt{(c-ab)^2 - 4(b^2-ac)(a^2-b)}}{2(b^2-ac)}=\\= -\sqrt{(c-ab)^2 - 4(b^2-ac)(a^2-b)}\leq0,
\end{align*}
so $F''(\alpha_+)\geq0$ and $\alpha_+$ is indeed the desired minimum. 

If $F''(\alpha_+)$ is zero, then $\alpha_+$ need not be a minimum. However, this is only the case if \[\sqrt{(c-ab)^2 - 4(b^2-ac)(a^2-b)} = 0.\] The computation showing the positivity of the discriminant suggests that in this case $(c-ab)^2 - 4(b^2-ac)(a^2-b) =  (c+2a^3-3ab)^2 + 4(b -a^2)^3 = 0$. This is a sum of two nonnegative numbers, so both numbers must be zero. However, $b-a^2$ has been shown to vanish only if $\gamma_k$ is an eigenvector of $A_k$. In this case also $c+2a^3 - 3ab = 0$, which shows that for $b^2\neq ac$, we have that $F''(\alpha_+)>0$, unless $\gamma_k$ is an eigenvector of $A_k$.

Now suppose that $b^2 = ac$. If $b = 0$, it follows that $g_k = 0$ and hence $F(\alpha)$ is zero everywhere, so regardless of the choice of $\alpha$, we cannot further decrease the objective function. In this case we choose $\alpha_+ = 0$. 
If $b\neq 0$, we have \[F'(\alpha) = -2 \frac{b- a^2 +\alpha(ab-c)}{(1-2\alpha a + \alpha^2 b^2)^2},\]
which vanishes for $\alpha_+ = \frac ab$.
The second derivative in $\alpha_+$  is given by \[F''(\alpha_+) = -2\frac{ab-c}{(1-2\tfrac{a^2}{b} +a^2)^2}.\]
Again, the denominator is positive and the numerator can be shown to be negative using a similar symmetrisation argument as in the proof of Theorem~\ref{thm:IAOLdecrease}. This argument also shows that $F''(\alpha_+) = 0$ if and only if $\gamma_k$ is an eigenvector of $A_k$.

\bibliography{karinbibtex}
\bibliographystyle{plain}

\end{document}